\documentclass[11pt,a4paper]{amsart}

\usepackage{a4wide}

\usepackage{bbm}

\usepackage{algorithm}
\usepackage{algpseudocode}

\usepackage{amsmath,amsfonts,amscd,amssymb,amsthm}
\usepackage{epsfig, graphicx}
\usepackage[foot]{amsaddr}

\usepackage[latin1]{inputenc}

\usepackage{bm}
\usepackage{amsmath}
\usepackage{amssymb}  

\usepackage{dsfont}

\usepackage[normalem]{ulem}

\usepackage{dsfont}

\DeclareMathOperator*{\Id}{I}
\DeclareMathOperator*{\VS}{V}

\newcommand{\lgen}{\ensuremath{\langle}}
\newcommand{\rgen}{\ensuremath{\rangle}}

\newcommand{\ls}{\ell}

\newcommand{\idS}{\ensuremath{\mathfrak{s}}}

\newcommand{\R}{\ensuremath{\mathds{R}}}		
\newcommand{\C}{\ensuremath{\mathds{C}}}		
\newcommand{\N}{\ensuremath{\mathds{N}}}		

\DeclareMathOperator*{\E}{\mathds{E}}					
\DeclareMathOperator*{\htid}{ht}  		        

\newcommand{\0}{\ensuremath{0}}

\newcommand{\cumu}{\kappa}

\newcommand{\rk}{\operatorname{rank}}

\newcommand{\codim}{\operatorname{codim}}

\newcommand{\lspan}{\operatorname{span}}

\newcommand{\calA}{\mathcal{A}}

\newcommand{\calF}{\mathcal{F}}
\newcommand{\calI}{\mathcal{I}}
\newcommand{\calJ}{\mathcal{J}}
\newcommand{\calM}{\mathcal{M}}

\newcommand{\calO}{\mathcal{O}}
\newcommand{\calP}{\mathcal{P}}

\newcommand{\fraks}{\mathfrak{s}}

\newcommand{\CC}{\mathbb{C}}

\newcommand{\NN}{\mathbb{N}}

\theoremstyle{plain}
\newtheorem{Satz}{Theorem}[section]
\newtheorem{Thm}[Satz]{Theorem}

\newtheorem{Prop}[Satz]{Proposition}
\newtheorem{Conj}[Satz]{Conjecture}
\newtheorem{Cor}[Satz]{Corollary}

\newtheorem{Lem}[Satz]{Lemma}

\newtheorem{Prob}[Satz]{Problem}

\theoremstyle{definition}
\newtheorem{Rem}[Satz]{Remark}
\newtheorem{Def}[Satz]{Definition}

\newtheorem{Ex}[Satz]{Example}
\newtheorem{Not}[Satz]{Notation}

\newtheorem{Ass}[Satz]{Assumption}

\def\itboxx#1{\ifvmode\indent\fi\makebox[2em][r]{\rmn(#1)} }
\newcommand{\rmn}{\fontshape{n}\fontseries{m}\selectfont\rm}

\begin{document}

\title{Regression for Sets of Polynomial Equations}

\author[F.~J.~Kir\'{a}ly]{Franz J.~Kir\'{a}ly$^1$}
\email{franz.j.kiraly@tu-berlin.de}
\address{$^1$Machine Learning Group \\
        Berlin Institute of Technology (TU Berlin) \\
        and Discrete Geometry Group, Institute of Mathematics, FU Berlin}

\author[P.~von B\"unau]{Paul von B\"unau$^2$}
\email{paul.buenau@tu-berlin.de}

\author[J.~S.~M\"uller]{Jan Saputra M\"uller$^2$}
\email{saputra@cs.tu-berlin.de}

\author[D.~A.~J.~Blythe]{Duncan A.~J.~Blythe$^3$}
\email{duncan.blythe@bccn-berlin.de}
\address{$^3$Machine Learning Group \\
         Berlin Institute of Technology (TU Berlin) \\
         and Bernstein Center for Computational Neuroscience (BCCN), Berlin}

\author[F.~C.~Meinecke]{Frank C.~Meinecke$^2$}
\email{frank.meinecke@tu-berlin.de}
\address{$^2$Machine Learning Group \\
         Berlin Institute of Technology (TU Berlin) }

\author[K.-R.~M\"uller]{Klaus-Robert M\"uller$^4$}
\email{klaus-robert.mueller@tu-berlin.de}
\address{$^4$Machine Learning Group \\
         Berlin Institute of Technology (TU Berlin) \\
	 	and IPAM, UCLA, Los Angeles, USA}

\begin{abstract}
We propose a method called \textit{ideal regression} for approximating an arbitrary system
of polynomial equations by a system of a particular type. Using techniques
from approximate computational algebraic geometry, we show how we can solve ideal regression
directly without resorting to numerical optimization. Ideal regression is useful
whenever the solution to a learning problem can be described by a system of
polynomial equations. As an example, we demonstrate how to formulate
Stationary Subspace Analysis (SSA), a source separation problem, in terms of
ideal regression, which also yields a consistent estimator for SSA. We then compare
this estimator in simulations with previous optimization-based approaches for SSA.
\end{abstract}

\maketitle

\section{Introduction}
\label{sec:intro}

Regression analysis explains
the relationship between covariates and a target
variable by estimating the parameters of a function which best fits the
observed data. The function is chosen to be of a particular type (e.g.\ linear)
to facilitate interpretation or computation. In this paper, we introduce a
similar concept to sets of polynomials equations: given arbitrary input
polynomials, the aim is to find a set of polynomials of a particular type that best
approximates the set of solutions of the input. As in ordinary regression, these
polynomials are parameterized to belong to a certain desired class. This class
of polynomials is usually somewhat simpler than the input polynomials. We call this
approach \textit{ideal regression}, inspired by the algebraic concept of
an ideal in a ring. In fact, the algorithm that we derive is based on techniques from
approximate computational algebraic geometry. In machine learning, ideal
regression is useful whenever the solution to a learning problem can be
written in terms of a set of polynomial equations. We argue that our ideal regression
framework has several advantages: (a)~it allows to naturally formulate regression
problems with intrinsical algebraic structure;
(b)~our algorithm solves ideal regression directly instead of resorting to less
efficient numerical optimization; and (c)~the algebraic formulation is amenable
to a wide range of theoretical tools from algebraic geometry. We will demonstrate
these points in an application to a concrete parametric estimation task.

\begin{figure*}[ht]
  \begin{center}
    \includegraphics{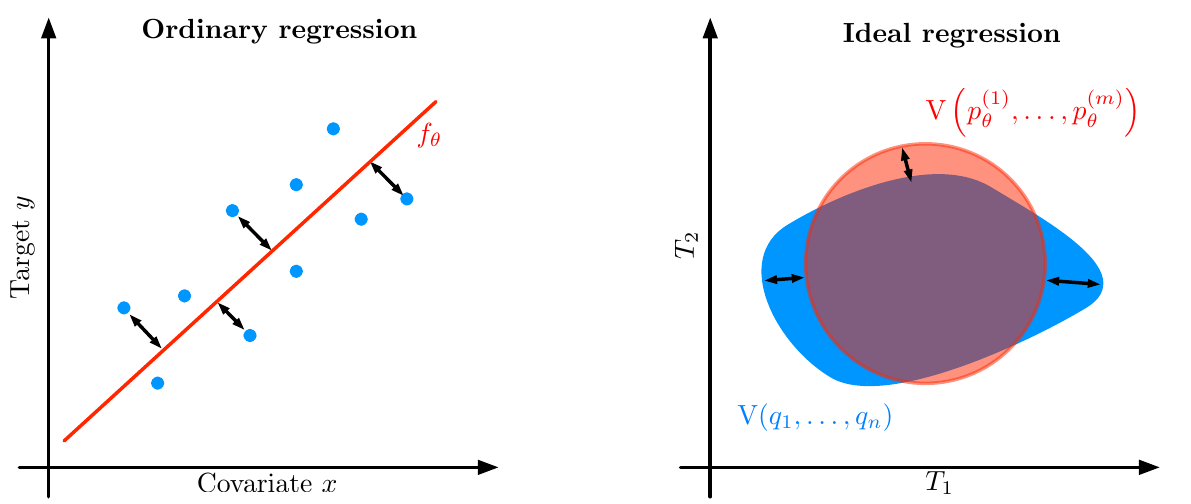}
 	 \caption{
	 	\label{fig:ideal_regression} Ordinary regression (left panel) fits the data (blue points)
		by a function $f_\theta$ (red line) parametrized by $\theta$. Ideal regression
		(right panel) approximates a set of arbitrary input polynomials $q_1, \ldots, q_n$ by
		a set of polynomials $p^{(1)}_\theta, \ldots, p^{(m)}_\theta$ that are of a special type
		parameterized by $\theta$. The approximation is in terms of their set of solutions:
			the parameter $\theta$ is chosen such that $\VS(p^{(1)}_\theta, \ldots, p^{(m)}_\theta)$
					(red shape) best fits the the solutions to the input (blue shape), in
					the space of $T_1$ and $T_2$.
			       }
  \end{center}
\end{figure*}

More formally, the analogy between ordinary and ideal regression is as follows. In
ordinary regression, we are given a dataset $\mathcal{D} = \{ (x_i, y_i) \}_{i=1}^n \subset \R^D \times \R$
and the aim is to determine parameters $\theta$ for a regression function $f_\theta : \R^D \rightarrow \R $
such that $f(x_i)$ fits the target $y_i$ according to some criterion, i.e.\ informally speaking,
\begin{align}
\label{eq:olsfit}
f_\theta (x_i) \approx y_i \text{ for all samples $(x_i, y_i) \in \mathcal{D}$}.
\end{align}
In ideal regression, the input consists of a set of $n$ arbitrary polynomial
equations $q_1(T) = \cdots = q_n(T) = 0$ in the vector of variables $T = (T_1, \ldots, T_D)$
which may e.g.\ correspond to the coordinates of data. That is, the data set
$\mathcal{D} = \{ q_1, \ldots, q_n \}$ consists of the coefficients of the input
polynomials. Let $\VS(\mathcal{D}) \subset \C^d$ be the set of solutions
to the input equations, i.e.\
\begin{align*}
	\VS(\mathcal{D}) = \{ x \in \C^D \; | \; q(x) = 0 \, \forall p \in \mathcal{D} \},
\end{align*}
where $p(x)$ denotes the evaluation of the polynomial $p$ on the values $x$.
The aim of ideal regression is to determine another set of polynomials
$p^{(1)}_\theta, \ldots, p^{(m)}_\theta$ parametrized by $\theta$ that best approximate
the input polynomials $\mathcal{D}$ in terms of their set of solutions;
that is, informally,
\begin{align}
\label{eq:irfit}
	\VS(p^{(1)}_\theta, \ldots, p^{(m)}_\theta) \approx \VS(\mathcal{D}) .
\end{align}
The class of polynomials (parametrized by $\theta$) by which we approximate arbitrary
input is chosen such that it has certain desirable properties, e.g.\ is easy to
interpret or is of a particular type prescribed by the context of the application.
Thus, in ordinary regression we fit a parametrized function to arbitrary data and in ideal
regression we fit a parametrized system of polynomial equations to arbitrary systems of polynomial equations. Note that even if $\VS(\mathcal{D})$ has no exact solution (e.g.~due to noise, over-determinedness, or when reducing degrees of freedom), we can still find an approximate regression system close to the inputs. Figure~\ref{fig:ideal_regression} illustrates the analogy between ordinary regression and
ideal regression.

The natural algebraic object to parameterize sets of equations up to additive and multiplicative ambiguities are ideals in polynomial rings\footnote{that is, sets of polynomials closed under addition in the set, and under multiplication with arbitrary polynomials}. The parametric family $p^{(1)}_\theta, \ldots, p^{(m)}_\theta$ corresponds to a parametric ideal $\mathcal{F}_{\theta}$ in the polynomial ring $\mathbb{C}[T_1,\dots, T_d].$ In ring theoretic language, the informal regression condition reformulates to
\begin{align}
\label{eq:irfitideal}
	\mathcal{F}_\theta \ni_{\text{approx}} \{q_1,\ldots, q_n\},
\end{align}
where $\ni_{\text{approx}}$ stands for being approximately contained. Intuitively, this means that the equations $q_i$ are well-approximated by the parametric ideal $\mathcal{F}_\theta.$

The ideal regression setting is fairly general. It can be applied to a wide range of
learning settings, including the following.
\begin{itemize}
\item {\bf Linear dimension reduction and feature extraction.} When linear features of data are known, ideal regression provides the canonical way to estimate the target parameter with or without ''independent'' or ''dependent'' labels. This subsumes PCA dimensionality reduction, linear regression with positive codimension and linear feature estimation.
\item {\bf Non-linear polynomial regression.} When the regressor is a set of polynomials with specific structure, as e.g.~in positive codimensional polynomial regression or reduced rank regression, ideal regression allows to estimate the coefficients for the regressor polynomials simultaneously.
\item {\bf Comparison of moments and marginals.} Ideal regression is the canonical tool when equalities or projections of cumulants are involved. The example we will pursue in the main part of the paper will be of this type, as it is possibly the simplest one where non-linear polynomials occur naturally.
\item {\bf Kernelized versions.} Non-linear feature mapping and kernelization is natural to integrate in the regression process as the presented estimator builds on least-squares estimates of vector spaces.
\end{itemize}

\begin{figure*}[ht]
  \begin{center}
    \includegraphics{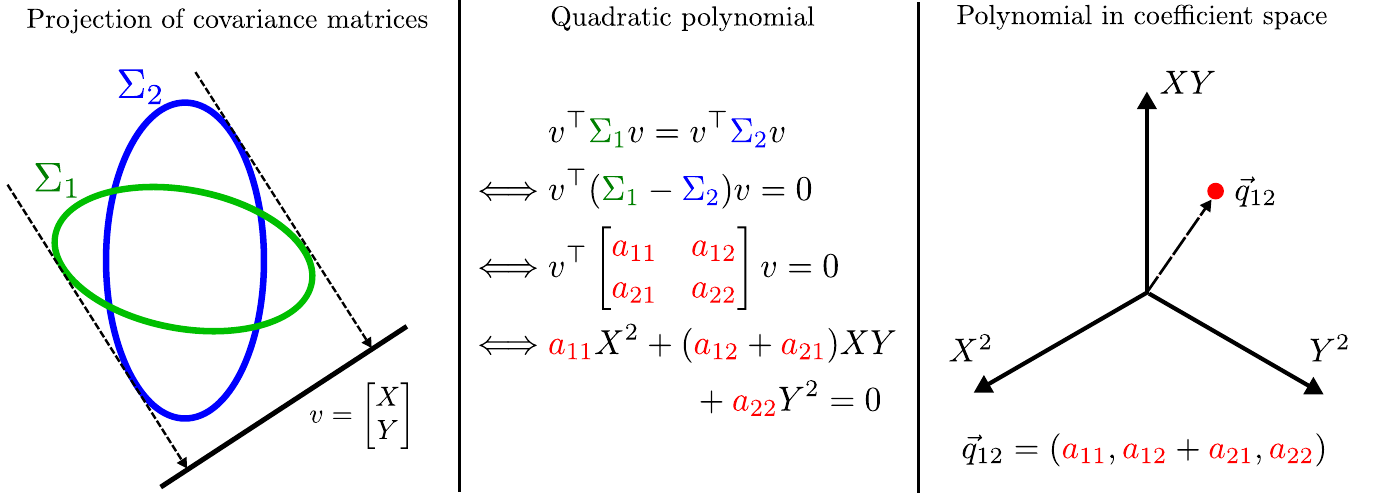}
 	 \caption{
	 	\label{fig:alg_setting}
			Representation of the problem: the left panel shows the covariance
			matrices $\Sigma_1$ and $\Sigma_2$ with the desired projection $v$.
			In the middle panel, this projection is defined as
			the solution to a quadratic polynomial. This polynomial is
			embedded in the vector space of coefficients spanned by the monomials
			$X^2, Y^2$ and $X Y$ shown in the right panel. 	
			       }
  \end{center}
\end{figure*}

\subsection{Example: finding common marginals}

In this paper, we will demonstrate ideal regression in a non-linear example: we will reformulate a statistical marginalization task as ideal regression. Namely, we study the following problem:

\begin{Prob}\label{Prob:marg}
Given $D$-variate random variables $X_1, \ldots, X_m,$
find a projection $P \in \R^{d \times D}$ to a $d$-dimensional subspace under which
which the $X_i$ are identically distributed, i.e.\
\begin{align*}
	P X_1 \sim \cdots \sim P X_m .
\end{align*}
\end{Prob}
For example, the $X_i$ can model different data clusters or epochs, for which we want to find an informative common projection $P$.
A special subcase of this problem, where the $X_i$ are approximated by Gaussians, is Stationary Subspace Analysis~\cite{PRL:SSA:2009}
which has been applied successfully to Brain-Computer-Interfacing~\cite{BunMeiSchMul10Finding}, Computer
Vision~\cite{MeiBunKawMul09Learning}, domain adaptation~\cite{HarKawWasBun10SSA},
geophysical data analysis and feature extraction for change point detection~\cite{BunMeiSchMul10Boosting}.
Previous SSA algorithms \cite{PRL:SSA:2009, HarKawWasBun10SSA, KawSamBunMei11AnInformation} have addressed this
task by finding the minimum of an objective function that measures the difference
of the projected cumulants on the sought-after subspace.

Under the assumption that such a linear map $P$ exists, we can describe the set of all maps yielding common marginals.
A necessary (and in practice sufficient) condition is that the projections of the cumulants under $P$ agree.
Thus the coefficients of the polynomial equations are given by the coefficients of
the cumulants of the $X_i$ (see Section~\ref{sec:marg-ir} for details).
The output ideals correspond to the possible row-spans of $P$; note that the fact that the $P X_i$ have identical distribution depends only on the row-span of $P$. Equivalently, the regression parameter $\theta$ ranges over the set all sub-vector spaces of dimension $d$ in $D$-space, i.e.~over the Grassmann manifold $\mbox{Gr}(d,D).$ The regression ideal $\mathcal{F}_\theta$ is then just $\mbox{I}(\theta),$ the ideal of the vector space $\theta$, considered as a subset of complex $D$-space.

\subsection{Outline of the algorithm}

In the application of ideal regression studied in this paper, the aim is to determine
linear polynomials that have approximately the same vanishing set as the input polynomials
derived from differences of cumulants. The representation in which the algorithm
computes is the vector space of polynomials: each polynomial is represented by a
vector of its coefficients as shown in the right panel of Figure~\ref{fig:alg_setting}. The
coefficient vector space is spanned by all monomials of a particular degree, e.g.\ in
Figure~\ref{fig:alg_setting} the axis correspond to the monomials $X^2, X Y$ and $Y^2$ because
all homogeneous polynomials of degree two in two variables can be written as a linear
combination of these monomials.

\begin{figure*}[ht]
  \begin{center}
    \includegraphics{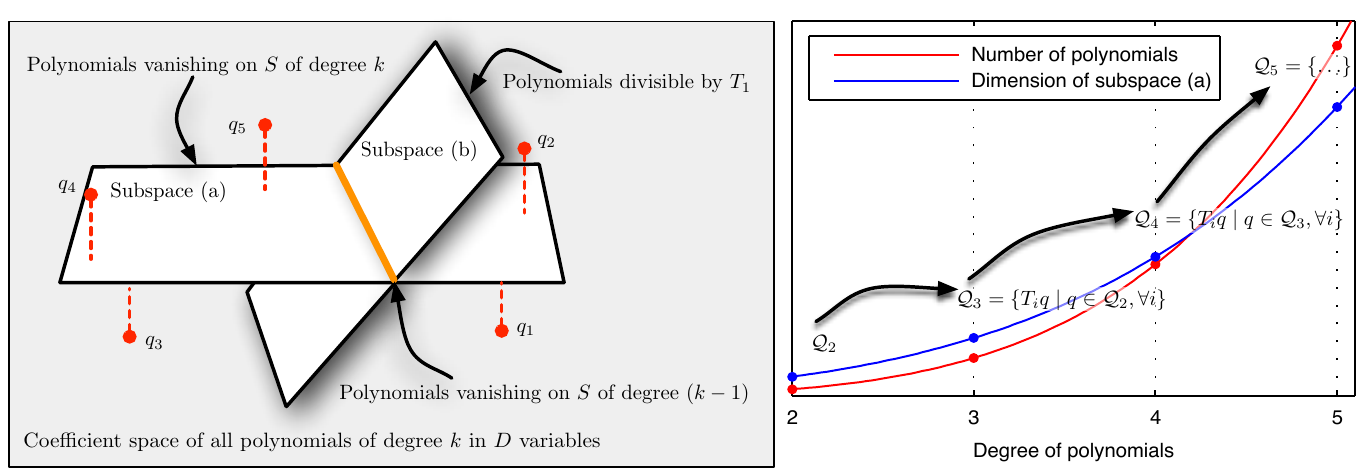}
 	 \caption{
	 	\label{fig:alg_multup}
			The left panel shows the vector space of coefficients; the input polynomials (red points)
			lie approximately on the subspace (a) of polynomials vanishing on $S$. In order to reduce
			the degree of the polynomial, we determine the intersection (orange) with the subspace (b)
			of polynomials that are divisible by a variable $T_1$. The polynomials in this
			intersection are isomorphic to the polynomials of degree $k-1$ that vanish on $S$.
			The right panel shows the M\"unchhausen process of multiplying the set of polynomials
			$\mathcal{Q}_i$ up to a degree so that we have enough
			elements $|\mathcal{Q}_i|$ to determine the basis for the subspace (a). In the shown case, where
			we start with $|\mathcal{Q}_2| = 5$ input polynomials in $D=6$ variables and
			$\dim S = 3$, we need to go up to degree 5. From $\mathcal{Q}_5$ we then descend to a linear
			form by repeatedly dividing out a single variable as shown in the left panel.
		 }
  \end{center}
\end{figure*}

The algorithm uses an algebraic trick: We first generate more and more equations by making the given ones more complicated; then, when their number suffices, we can make them simpler again to end at a system of the desired type. Playing on analogies, we thus call the algorithm the \textit{M\"unchhausen}
procedure\footnote{After the eponymous and semi-fictional Baron Hieronymus Carl Friedrich M\"unchhausen, who purportedly pulled himself and his horse out of the swamp by his own hair; compare the M\"unchhausen trilemma in epistemiology.}.

In the first part of the M\"unchhausen algorithm, we generate the said larger system of equations, having a particular (higher) degree,
but the same vanishing set. This is done by multiplying the input equations by all monomials
of a fixed degree, as illustrated in the right panel of Figure~\ref{fig:alg_multup}.
We will show that for any (generic) input, there always exists a degree such that the
resulting polynomials span a certain linear subspace of the coefficient
vector space (up to noise), i.e.\ that the red curve will always cross the blue
curve. That is, we show that the coefficient vectors lie approximately
on a linear subspace of known dimension (subspace~(a) in Figure~\ref{fig:alg_multup}).
This is the case because we have assumed that there exists a subspace of the data space
on which the marginals agree. After this, we obtain an approximate
basis for this subspace in coefficient space by applying PCA dimension reduction; this provides
us with an approximate basis for subspace~(a) in the left panel of Figure~\ref{fig:alg_multup}.

In the second part of the algorithm, we reduce the degree of the system of equations,
i.e.\ the approximate basis for subspace~(a), by repeatedly dividing out single
variables\footnote{i.e.~saturating with the homogenizing variable}.
In each step, we reduce the degree by one as illustrated by
the left panel of Figure~\ref{fig:alg_multup}. We compute a basis for the
intersection (orange line) of subspace~(a) and the subspace~(b) of polynomials of degree
$k$ that are divisible by the variable $T_1$. By dividing each basis element by $T_1$
we have obtained a system of equations of lower degree, which has approximately the same
set of solutions as the input. We repeat this process until we arrive at a system
of linear equations.

\subsection{Relation to other work}
The ideal regression approach draws inspiration from and integrates several concepts from different fields of research.
The first important connection is with computational algebra, as the estimation procedure is essentially a ring theoretic algorithm which can cope with noise and inexact data. In the noiseless case, estimating the regression parameter $\theta$ is essentially calculating the radical of a specific ideal, or, more specifically, computing the homogenous saturation of an ideal. These tasks are notoriously known to be very hard in general\footnote{namely, doubly exponential in the number of variables, see e.g.~\cite[section 4.]{Lap06Alg}}; however, for generic inputs, the computational complexity somewhat drops into feasible regions, as implied by the results on genericity in the appendix. The best known algorithms for computations of radicals are those of \cite{GiaTraZac88Gro}, implemented in AXIOM and REDUCE, the algorithm of \cite{EisHunVas92Dir}, implemented in Macaulay 2, the algorithm of \cite{CabConCar97Yet}, currently implemented in CoCoA, and the algorithm of \cite{KriLog91Alg} and its modification by \cite{Lap06Alg}, available in SINGULAR. Closely related to homogenous saturation is also the well-known Buchberger's algorithm for computation of reduced Gröbner basis, which can be seen as the inhomogenous counterpart of homogenous saturation when a degree-compatible term order is applied.

The second contribution comes from the field of approximate and numerical algebra, as the exact algorithms from computational algebra are numerically unstable even under small variations of the inputs and thus unfeasible for direct application to our case. The first application of numerical linear algebra to vector spaces of polynomials can be found in \cite{Cor95}, the numerical aspects of noisy polynomials have been treated in \cite{Ste04}. Also, the nascent field of approximate algebra has developed tools to deal with noise, see \cite{KrePouRob09}. In particular, the approximate vanishing ideal algorithm \cite{Hel06} fits polynomial equations to noisy data points with a method that essentially applies a sequence of weighted polynomial kernel regressions. The estimator for ideal regression given in this paper is essentially a deterministic algorithm using approximate computational algebra.

On the conceptual level, the idea of using (linear and commutative) algebra as an ingredient in statistical modelling and for solution of statistical problems is natural, as algebra is the science of structure. Therefore, studying structure in a statistical context makes algebra under stochastical premises a canonical tool. This idea gives rise to a plethora of different approaches, which today are subsumed as the field of algebraic statistics - in the broader meaning of the term. Standard references in the field include \cite{Stu02}, \cite{Stu10} and \cite{Gib10}, or \cite{Ama00} and \cite{Wat09}. Also, there is a range of machine learning methods dealing with non-linear algebraic structure or symmetries in data, e.g.~\cite{Ris07Skew,RisBor08Skew} or \cite{GPCA05}. In many applications, the concept of genericity also arises as the algebraic counterpart of statistical nondegeneracy; where it is mostly applied to the choice of model parameters and the study of a general such. Also, algebraic geometry and commmutative algebra are already successfully applied there to obtain theoretical results on statistical objects and methods.

\section{The algorithm}\label{sec:marg-ir}

The probability distribution of every smooth real random variable $X$ can be fully characterized in terms
of its \textit{cumulants}, which are the tensor coefficients of the cumulant generating function.

Before we continue, we introduce a useful shorthand notation for linearly transforming
tensors, i.e.~cumulants.
\begin{Def}
Let $A\in \CC^{d\times D}$ be a matrix. For a tensor $T\in \R^{D^{(\times k)}},$ we will denote
by $A\circ T$ the application of $A$ to $T$ along all tensor dimensions, i.e.
\begin{align*}
\left(A\circ T\right)_{i_1\dots i_k}=\sum_{j_1=1}^D\dots \sum_{j_k=1}^D A_{i_{1}j_{1}}\cdot\ldots\cdot A_{i_{k}j_{k}}T_{j_1\dots j_k}.
\end{align*}
\end{Def}
Using this, we can now define the cumulants of a $D$-dimensional smooth real random variable $X$ via the Taylor expansion of the cumulant generating function.
\begin{Def}
Let $X$ be a smooth real $D$-dimensional random variable. Then the cumulant generating function of $X$ is defined as
$
	\chi_X(\tau)  = \log \left( \E \left[ e^{i \tau^\top X} \right] \right)
				 = \sum_{k=1}^\infty (i \tau) \circ \frac{\cumu_k(X)}{k!},
$
where $\tau\in \R^D.$ The coefficient tensors $\cumu_k(X)$ are called the $k$-th cumulants of $X$.
\end{Def}
For the problem addressed in this paper, cumulants are a particularly suitable
representation because the cumulants of a linearly transformed random variable
are the multilinearly transformed cumulants, as a classical and elementary calculation shows:
\begin{Prop}
Let $X$ be a smooth real $D$-dimensional random variable and let $A \in \R^{d \times D}$ be a
matrix. Then the cumulants of the transformed random variable $A\cdot X$ are the transformed
cumulants,
$
	\cumu_k(A X) = A  \circ \cumu_k(X)
$
where $\circ$ denotes the application of $A$ along all tensor dimensions.
\end{Prop}
We now derive an algebraic formulation for Problem~\ref{Prob:marg}: note that
$P X_i \sim P X_j$ if and only if $v X_i\sim v X_j$ for all row vectors $v\in \lspan P^\top.$
\begin{Prob}
Find all $d$-dimensional linear subspaces in the set of vectors
\begin{align*}
	S & = \{ v \in \R^D \; \left| \; v^\top X_1 \sim \cdots \sim v^\top X_m \} \right. \\
	  & = \{v\in \R^D \; \left| \; v^\top \circ \cumu_k (X_i)=v^\top \circ \cumu_k (X_j),\; \right.\\
	  & \hspace{4cm}
	  k \in\N, 1 \le i,j\le m\}   \\
  	  & = \{v\in \R^D \; \left| \; v^\top \circ ( \cumu_k (X_i) - \cumu_k (X_m) ) = 0 ,\; \right. \\
	  & \hspace{4cm}
	  k \in\N, 1 \le i < m\}.
\end{align*}
\end{Prob}
The equivalence of the problems then follows from the fact that the projection $P$ can be characterized by its
row-span which is a $d$-dimensional linear subspace in $S$. Note that while we are looking for linear subspaces in $S$,
in general $S$ itself is not a vector space. Apart from the fact that $S$ is homogeneous, i.e.~$\lambda S = S$ for all
$\lambda \in \R\setminus \{0\}$, there is no additional structure that we utilize.
To use the tools from computational algebra, we now only need to consider the
left hand side of each of the equations as polynomials $f_1, \ldots, f_{n}$ in the variables $T_1, \ldots, T_D$,
\begin{align*}
	f_j = \begin{bmatrix} T_1  \cdots  T_D \end{bmatrix} \circ  ( \cumu_k (X_i) - \cumu_k (X_m) ),
\end{align*}
with $j$ running through $n$ combinations of $i$ and $k$. The $f_j$ are formally elements of the polynomial ring over the complex numbers $\C[T_1,\dots, T_D]$. In particular, if we restrict ourselves to a finite number of cumulants, we can write $S$ as the set of solutions to
a finite number $n$ of polynomial equations,
\begin{align*}
	S  = \left\{ v \in \R^D \; \left| \;  f_{1}(v) = \cdots = f_{n}(v) = 0, \; 1 \le j \le k \right\} \right.,
\end{align*}
which means that $S$ is an \textit{algebraic set} or an \textit{algebraic variety}.
Thus, in the language of algebraic geometry, we can reformulate the problem as follows.
\begin{Prob}\label{Prob:Alg}
Find all $d$-dimensional linear subspaces in the algebraic set
\begin{align*}
	S = \VS( f_1, \ldots, f_{n}) .
\end{align*}
\end{Prob}
In order to describe in algebraic terms how this can be done we need to assume that a unique
solution indeed exists, while assuming as little as possible about the given polynomials.
Therefore we need to employ the concept of \textit{generic polynomials}, which is defined
rigorously in the supplemental material. Informally, a polynomial is generic when it does
not fulfill any other conditions than the imposed ones and those which follow logically. This
can be modelled by assuming that the coefficients are independently sampled from a sufficiently general probability distribution (e.g.~Gaussians),
only subject to the imposed constraints. The statements following the assumption of genericity are then probability-one statements
under those sampling conditions.

In our context, we can show that under genericity conditions on the $f_1,\dots, f_n$,
and when their number $n$ is
large enough, the solution is unique and equivalent to finding a linear generating set for the
radical of the ideal $\langle f_1,\dots, f_{n}\rangle$, which equals its homogenous saturation,
as Corollary~\ref{Cor:IdentS} in the supplemental
material asserts:
\begin{Prob}
\label{prob:finalpaper}
Let $f_1,\dots, f_n$ with $n \geq D+1$ be generic homogenous polynomials vanishing on a
linear $d$-dimensional subspace $S\subseteq \C^D$. Find a linear generating set $\ls_1, \ldots, \ls_{D-d}$ for the
radical ideal
\begin{align*}
	(\langle f_1, \ldots, f_{n}\rangle : T_D) = \sqrt{\langle f_1, \ldots, f_n\rangle}=\Id (S).
\end{align*}
\end{Prob}

\section{Approximate Algebraic Computations}

In this section we present the Münchhausen algorithm which computes the homogenous saturation in Problem~\ref{prob:finalpaper} and thus computes the ideal regression in the marginalization problem. The algorithm for the general case is described in the appendix in section~\ref{sec:appsat}.

The efficiency of the algorithm stems largely from the fact that we operate with
linear representations for polynomials. That is, we first find enough elements in the
ideal $\lgen f_1, \ldots, f_n \rgen$ which we then represent in terms of coefficient
vectors. In this vector space we can then find the solution by means of linear
algebra. An illustration of this representation is shown in Figure~\ref{fig:alg_setting}.
Let us first introduce tools and notation.
\begin{Not}
We will write $R=\C[T_1,\dots, T_D]$ for the ring of polynomials over $\C$ in the variables $T_1,\dots, T_D.$ We will denote the ideal of the $d$-dimensional linear space $S\subseteq\C^D$ by $\idS = \Id(S).$
\end{Not}
In order to compactly write sub-vector spaces of certain degree, we introduce some notation.
\begin{Not}
Let $\calI$ be an ideal of $R$. Then we will denote the vector space of homogenous polynomials of degree $k$ in $\mathcal{I}$ by $\calI_k.$
\end{Not}
The dimension of these sets can be later written compactly in terms of simplex numbers, for which we introduce an abbreviating notation.
\begin{Not}
We denote the $b$-th $a$-simplex number by
$\Delta (a,b)={a+b-1 \choose a}$, which is defined to be zero for $a < 0$.
\end{Not}

Since the polynomials arise from estimated cumulants we need to carry all algebraic computations
out approximately. The crucial tool is to minimize distances in coefficient space using
the singular value decomposition (SVD).
\begin{Def}
Let $A\in \C^{m\times n}$ be a matrix, let $A=UDV^\top$ be its SVD. The approximate row span of $A$ of rank $k$ is the row span of the first $k$ rows of $V$; the approximate left null space of $A$ of rank $k$ is the row span of the last $k$ rows of $U$.
\end{Def}

These approximate spaces can be represented by matrices consisting of row vectors spanning them, the so-called approximate left null space matrix and approximate row span matrix.

The key to the problem is the fact that there exists
a degree $N$ such that the ideal generated by the $f_1,\dots, f_n$ contains all homogenous
polynomials of degree $N$ in $\fraks$. This allows us to increase the degree
until we arrive at a vector space where it suffices to operate linearly (see Figure~\ref{fig:alg_multup}).
\begin{Thm}\label{Prop:AlgProp}
Let $f_1,\dots, f_n\in \fraks$ be generic homogenous polynomials in $D$ variables of fixed degrees $d_1,\dots, d_n$ each, such that $n>D.$ Let $\calI=\langle f_1,\dots, f_n\rangle$. Then one has
$$(\calI:T_i)=\fraks$$
for any variable $T_i$. In particular, there exists an integer $N$ such that
$$\calI_N=\fraks_N.$$
$N$ is bounded from below by the unique index $M$ belonging to the first non-positive coefficient $a_M$ of the power series
$$\sum_{k=0}^\infty a_kt^k=\frac{\prod_{i=1}^n (1-t^{d_i})}{(1-t)^D}-\frac{1}{(1-t)^d}.$$
If we have $d_i\le 2$, and $D\le 11$, then equality holds.
\end{Thm}
This summarizes Proposition~\ref{Prop:multterm}, Corollary~\ref{Cor:Nbound} and Theorem~\ref{Thm:Frocheck} from the supplemental material for the case $\fraks=\Id (S)$, the proof can be found there.
In the appendix, we conjecture that the statement on $N$ is also valid for general $d_i$ and $D$ - this generalizes Fröberg's conjecture on Hilbert series of semi-regular sequences \cite{Fro94}. In the supplemental material, we give an algorithm which can be used to prove the conjecture for fixed $d_i$ and $D$ and thus give an exact bound for $N$ in those cases. Theorem~\ref{Prop:AlgProp} guarantees that given our input polynomials, we can easily obtain
a basis for the vector spaces of homogenous polynomials $\fraks_k$ with $k\ge N$.
In this vector space, we are interested in the polynomials divisible by a fixed monomial $T_i$; they form the vector space $\fraks_k \cap \langle T_i \rangle = \left(\fraks\cap\langle T_i\rangle \right)_k.$
By dividing out the monomials $T_i$, we can then obtain the vector space of homogenous polynomials $\fraks_{k-1}$ of degree one less.

\begin{algorithm}[h]
\caption{\label{Alg:rad-ssa-approx} The ideal regression estimator.
\textit{Input:} Generic homogenous polynomials $f_1,\dots, f_n\in\fraks, n\ge D$; the dimension $d$ of $S$.
\textit{Output:} An approximate linear generating set $\ls_1, \ldots, \ls_{D-d}$ for the ideal $\idS$. }
\begin{algorithmic}[1]
    \State Determine some necessary degree $N$ according to Theorem~\ref{Prop:AlgProp}
          (e.g. with Algorithm~\ref{Alg:FroN}, or via Conjecture~\ref{Conj:Frogen}/Algorithm~\ref{Alg:Frogen})
			 \label{alg:rad_line1}

    \State Initialize $Q \gets [\,]$ with the empty matrix.

    \For{$i=1\dots n$}
       \For{all monomials $M$ of degree $N-\deg f_i$}
            \State \label{alg:rad_line2}

            	Add a row vector of coefficients, $Q \gets \begin{bmatrix}  Q \\ f_i M \end{bmatrix}$

       \EndFor
    \EndFor

	\For{$k=N\dots 2$} \label{alg:rad_line3}
		\State \label{alg:rad_line4} Set $Q \gets \mbox{ReduceDegree(}Q\mbox{)}$
	\EndFor \label{alg:rad_line5}

    \State Compute the approximate row span matrix
	 	$\mbox{$A \gets \begin{bmatrix} a_1  \cdots  a_{D-d} \end{bmatrix}^\top$ of $Q$ of rank $D-d$ \label{alg:rad_line6}}$

	\State 	Let $\ls_i \gets \begin{bmatrix} T_1 & \cdots & T_D \end{bmatrix} a_i $ for all $1 \leq i \leq D-d$   	  	
\end{algorithmic}
\end{algorithm}

We explain how to proceed in the case where $\fraks=\Id (S)$ is linear, and the $f_i$ are generic.
The case of general $\fraks$ can be found in the appendix.
So suppose that we have enough quadratic polynomials. Then we can determine an approximate basis
for the vector space of homogenous degree $2$ polynomials vanishing on $S$ which is $\fraks_2$, because our input polynomials
lie approximately on that subspace. From this we can obtain the linear homogenous polynomials $\fraks_1$ by dividing out
any monomial $T_i$; a basis of $\fraks_1$ can be directly used to obtain a basis of $S$.
More generally, if we know a basis of the vector space of homogenous degree $k$ polynomials vanishing on $S$ which is $\fraks_{k},$
we can obtain from it a basis of $\fraks_{k-1}$
in a similar way; by repeating this stepwise, we eventually arrive at $\fraks_1$.
This degree reducing procedure is approximatively applied in Algorithm~\ref{Alg:ReduceDeg}.

We now describe the M\"unchhausen Algorithm~\ref{Alg:rad-ssa-approx} in detail. In Step~\ref{alg:rad_line1}, we calculate a degree $N$ up to which we need to multiply our input polynomials so that we have enough
elements to approximately span the whole space $\fraks_N$ of polynomials
vanishing on $S$. In the Steps~2~to~7 we multiply the
input polynomials up to the necessary degree $N$ and arrange them as coefficient vectors in the matrix $Q$.
This process is illustrated in the right panel of Figure~\ref{fig:alg_multup}.
The rows of $Q$ approximately span the space of polynomials vanishing on $S$, i.e.~space (a) in the
left panel of Figure~\ref{fig:alg_multup}. In Steps~8~to~10 we invoke Algorithm~\ref{Alg:ReduceDeg} to reduce
the degree of the polynomials in $Q$ until we have reached an approximate linear representation
$\fraks_1$. Finally, in Step~11 we reduce the set of linear generators in $Q$ to the principal $D-d$ ones.

\begin{algorithm}[h]
\caption{\label{Alg:ReduceDeg} ReduceDegree ($Q$).
\textit{Input:} Approximate basis for the vector space $\fraks_k,$
given as the rows of the ($n\times \Delta (n,D)$)-matrix $Q$; the dimension $d$ of $S$.
\textit{Output:} Approximate basis for
the vector space $\fraks_{k-1},$
given as the rows of the ($n'\times \Delta (n-1,D)$)-matrix $A$}
\begin{algorithmic}[1]		
	\For{$i=1\dots D$} \label{alg:red_line1}
		\State \label{alg:red_line2}
		\begin{minipage}[t]{11cm}
			Let $Q_i$ $\gets$ the submatrix of $Q$ obtained by\\ removing all
    		columns corresponding to\\ monomials divisible by $T_i$
		\end{minipage}
		
		\State \label{alg:red_line3}
		\begin{minipage}[t]{11cm}
			Compute $L_i$ $\gets$ the\\ approximate left null space matrix of $Q_i$\\
            of rank
            $m-\Delta(k,D)+\Delta(k,d)$\\ \phantom{of rank} $+\Delta(k-1,D)-\Delta(k-1,d)$
		\end{minipage}
				
		\State \label{alg:red_line4}
		\begin{minipage}[t]{11cm}
            Compute $L'_i \gets$ the\\ approximate row span matrix of $L_i Q$\\ of rank $\Delta(k-1,D)-\Delta(k-1,d)$	
        \end{minipage}
		
		\State \label{alg:red_line5}
		\begin{minipage}[t]{11cm}
		Let $L''_i \gets $ the matrix obtained from $L'_i$ by\\ removing all columns corresponding to\\
				monomials not divisible by $T_i$
		\end{minipage}
	\EndFor \label{alg:red_line6}

    \State \label{alg:red_line7} Let $L \gets$ the matrix obtained by vertical concatenation of all $L''_i$

    \State \label{alg:red_line8}
    	\begin{minipage}[t]{11cm}
           Compute $A \gets $ the\\ approximate row span matrix of
	 	   $L$ of rank \\
           $n'=\min (n,D(\Delta(k-1,D)-\Delta(k-1,d)))$
       \end{minipage}
  	  	
\end{algorithmic}
\end{algorithm}

\begin{figure*}[htb]
  \begin{center}
    \includegraphics{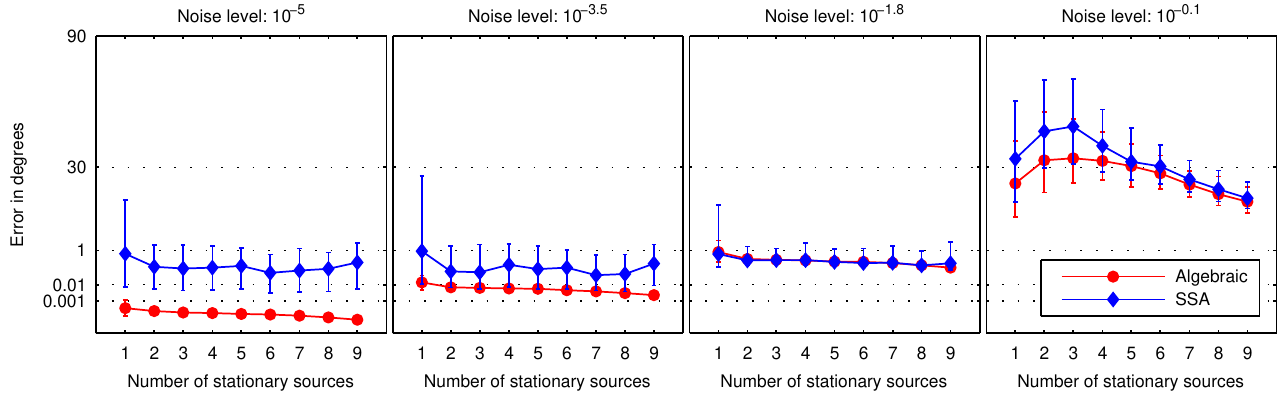}
  \caption{
        \label{fig:alg_nips}
        Comparison of the ideal regression algorithm and the SSA algorithm. Each panel shows
        the median error of the two algorithms (vertical axis) for varying numbers of
		stationary sources in ten dimensions (horizontal axis). 
		The noise level increases
		from the left to the right panel; the error bars extend from the 25\% to the 75\%
		quantile. 
       }
  \end{center}
\end{figure*}

Given a set of polynomials of degree $k$ that vanish approximately $S$, Algorithm~\ref{Alg:ReduceDeg}
computes another set polynomials of degree $k-1$ with the same property. This is achieved
by dividing out variables \textit{approximately}, in a way that utilizes as much information as
possible to reduce the influence of estimation errors in the coefficients of the input polynomials.
Approximate division by a variable $T_i$ means that we find linear combinations of our input polynomials
such the coefficients of all monomials not divisible by $T_i$ are as small as possible. Given a matrix
of coefficient row vectors $Q$ of degree $k$, for each variable that we divide out the result is
a matrix $L''_i$ of polynomials of degree $k-1$ that also vanish approximately on the set of
solutions $S$. We iterate over all variables in the for-loop and combine the results
$L''_1, \ldots, L''_D$ in the final Steps~7 and 8. In order to find $L''_i$ for each variable $T_i$, we first determine a matrix $L_i$ that minimizes
the coefficients for all monomials in $Q$ that are not divisible by $T_i$. To that end, in Step~2
we remove all monomials not divisible by $T_i$ to get a matrix $Q_i$ for which we then compute
the left null space matrix $L_i$ in Step~3. The product $L_i Q_i$ is then a set of
polynomials of degree $k$ with minimal coefficients for all monomials not divisible by $T_i$.
These polynomials lie approximately on the span of polynomials vanishing on $S$. In the
next Step~4, we compute an approximate basis $L'$ for this space and in Step~5 we reduce
the degree by removing all monomials not divisible by $T_i$. Finally, in the last Steps~7 and 8 we combine all found solutions $L''_1, \ldots, L''_D$ using PCA.
Note that both algorithm are deterministic and consistent estimators.

\section{Simulations}

We investigate the influence of the noise level and the number of
dimensions on the accuracy and the runtime of the ideal regression algorithm in the special case of covariance matrices and compare it to the standard method for this case, the SSA algorithm \cite{PRL:SSA:2009}.
We measure the accuracy using the subspace angle between the true and the estimated space of projections.
The setup of the synthetic data is as follows: we fix the total number of dimensions to $D=10$ and
vary the dimension $d$ of the subspace with equal probability distribution from
one to nine. We also fix the number of random variables to $m=26,$ yielding $n=25$ quadratic polynomials. For each trial of the simulation, we choose a random basis for the subspace $S$ and random covariance
matrices to which we add a disturbance matrix parametrized by the noise level
$\sigma$.

The results are shown in Figure~\ref{fig:alg_nips}. With increasing noise levels
 both algorithms become worse. For all noise levels, the algebraic method
yields significantly better results than the standard optimization-based approach, over all dimensionalities.
In Figure~\ref{fig:alg_nips}, we see that the error level of the ideal regression algorithm decreases with the noise level, converging to the exact solution when the noise tends to zero. In contrast, the error of original SSA decreases with noise level, reaching a minimum error baseline which it cannot fall below. In particular, the algebraic method significantly outperforms SSA for low noise levels, approaching machine precision. At high noise levels, the algebraic method outperforms SSA on average, having lower error variance than SSA.

\section{Conclusion}  

In this paper, we have presented the framework of ideal regression and an estimator which uses approximate computational algebra. Moreover, we have worked through a specific example: we have shown that the problem of finding common projections of marginals can be reformulated in terms of ideal regression and we have derived
a practical algorithm, that we have evaluated in numerical simulations. Also, due to the algebraic formulation of the problem, we were able to derive previously unapproachable theoretical results on the estimation problem.
We argue for a cross-fertilization of machine learning and approximate computational algebra: the former can benefit from the wealth of techniques for dealing with uncertainty and noisy data; the machine learning community may find in ideal regression a novel framework for representing learning problems and
powerful proof techniques.

\subsubsection*{Acknowledgments}

We thank Gert-Martin Greuel and Bernd Sturmfels for insightful discussion and we gratefully
acknowledge support by the Mathematisches Forschungsinstitut Oberwolfach (MFO).

\appendix
\section*{APPENDIX -- SUPPLEMENTARY MATERIAL}

This Appendix contains the theoretical background for a treatment of the ideal regression problem. In Section~\ref{sec:eqs-idr} we explain why ideal regression is the correct framework to estimate parametric systems of equations. In Section~\ref{app-generic}, we present the concept of genericity, which is essential for formulating a generative model of ideal regression. Moreover, we prove some theoretical results related to genericity. These results are then applied in Section~\ref{sec:appidr} to obtain identifiability results and an estimator for ideal regression, and - as a special application - for the common marginals problem.

\section{From estimating sets of equations to ideal regression}
\label{sec:eqs-idr}
In this section, we will explain why ideal regression is the natural formulation for estimating sets of equations. We will provide some examples leading to the conclusion that ideals in rings are the canonical objects which capture the ambiguities of sets of polynomial equations. The reader may find some knowledge on ring theory, in particular on ideals and Hilbert's Nullstellensatz helpful, as presented for example in \cite{Cox}, but not necessary to understand the phenomena presented in this section.

As already stated in the main corpus of the paper, we want to estimate a system of polynomial equations with specific structure, given some arbitrary system of polynomial equations:
\begin{Prob}\label{Prob:eqsreg}
Given input polynomials $q_1,\dots, q_n$, estimate a regression parameter $\theta$ such that the parametric system $p^{(1)}_\theta,\dots, p^{(m)}_\theta,$ is ``close'' to the inputs $q_1,\dots, q_n.$
\end{Prob}
Of course, Problem~\ref{Prob:eqsreg} is still an informal problem description: it remains to state how $\theta$ parameterizes the system of equations, and it has yet to be stated what ``close'' should mean. Intuitively, ``close'' should mean that the set of solutions defined by the $p^{(1)}_\theta,\dots, p^{(m)}_\theta$ is close to the set or sets of solutions defined by the $q_1,\dots, q_n$, i.e.~formally
\begin{align*}
V_\theta &=\VS (p^{(1)}_\theta,\dots, p^{(m)}_\theta)\\
& =\{x\in\mathbb{C}^D\;;\;p^{(1)}_\theta(x)=\dots = p^{(m)}_\theta(x)=0\}.
\end{align*}
Before we continue, we first show a basic example for ideal regression:

\begin{Ex}\rm
In ordinary regression, we are given points $(x^{(1)},y^{(1)}),\dots, (x^{(N)},y^{(N)})\in\mathbb{R}^k\times \mathbb{R}.$ We want to estimate a linear polynomial
$$p_\theta(X_1,\dots, X_k,Y) = \beta_1 X_1 +\dots+\beta_k X_k + \alpha - Y$$
with parameter $\theta = (\beta_1,\dots, \beta_k,\alpha)\in\mathbb{R}^{k+1}$
such that $p_\theta(x^{(i)},y^{(i)})$ is small for all $i$. For example, in least squares regression, the optimal $\theta$ is obtained by minimizing the sum of the squares of the $p_\theta(x^{(i)},y^{(i)})$.

Now each point $(x^{(i)},y^{(i)})$ is the unique solution to the set of $k+1$ equations
\begin{align*}
q_{i0}(X_1,\dots, X_k,Y)&= Y_{\phantom{0}} - y^{(i)}=0\\
q_{i1}(X_1,\dots, X_k,Y)&= X_1 - x^{(i)}_1=0\\
q_{i2}(X_1,\dots, X_k,Y)&= X_2 - x^{(i)}_2=0\\
&\vdots\\
q_{ik}(X_1,\dots, X_k,Y)&= X_k - x^{(i)}_k=0
\end{align*}
For one point $(x^{(i)},y^{(i)})$, being close to the regression hyperplane $\VS(p_\theta)$ means that $p_\theta$ is a principal vector of the $q_{ij}$ with $i$ fixed (w.r.t certain error measures) considered as elements in the vector space of linear polynomials. So, being a good approximation to all points means that $p_\theta$ is a principal vector for the data given by all $q_{ij}.$
\end{Ex}
The following examples will show the central ambiguities which occur when considering more than one equation:
\begin{Ex}\rm
Let us imagine we want to regress two equations instead of a single one, i.e. we want to determine two regressor polynomials
\begin{align*}
p^{(1)}_\theta(X_1,\dots, X_k,Y) = \beta^{(1)}_1 X_1 +\dots+\beta^{(1)}_k X_k + \alpha^{(1)} - Y\\
p^{(2)}_\theta(X_1,\dots, X_k,Y) = \beta^{(2)}_1 X_1 +\dots+\beta^{(2)}_k X_k + \alpha^{(2)} - Y
\end{align*}
where $\theta$ includes the information on the regression coefficients $\beta^{(i)}_j$ and the $\alpha^{(i)}.$ Now it is essential to note that the set of solutions
\begin{align*}
V_\theta&=\VS(p^{(1)}_\theta,p^{(2)}_\theta)\\
&=\{(x,y)\in\mathbb{R}^{k}\times\mathbb{R}\;;\;p^{(1)}_\theta(x,y)=p^{(2)}_\theta(x,y)=0\}
\end{align*}
is already uniquely determined by the linear span of the two polynomials $p^{(1)}_\theta(X_1,\dots, X_k,Y)$ and $p^{(2)}_\theta(X_1,\dots, X_k,Y),$ seen as elements in the vector space of linear polynomials. For example, two polynomials $p^{(1)}_\theta$ and $p^{(2)}_\theta$ give rise to the same set of solutions $V_\theta$ as the two polynomials $p^{(1)}_\theta + p^{(2)}_\theta$ and $p^{(1)}_\theta-p^{(2)}_\theta.$ One can also prove that these are the only ambiguities in the solution. Thus for $V_\theta$, the parameter $\theta$ has not $2k+2$ degrees of freedom, but only $2k+1$. The parameter space for $\theta$ is the space of all $2$-dimensional affine linear spaces in $(k+1)$-space. In general, similar additive ambiguities occur, and it makes sense to speak about the set of solutions only uniquely with respect to these additive symmetries.
\end{Ex}
\begin{Ex}\rm
Similarly, consider the case where we want to regress a conic section, i.e.~we want to regress a linear polynomial $\ell_\theta$ and a quadratic polynomial $q_\theta;$ the parameter $\theta$ determining all coefficients of the two polynomials. Again, different choices of coefficients can lead to the set of solutions
$$V_\theta=\VS(\ell_\theta,q_\theta);$$
For example, if $\ell_\theta, q_\theta$ are some choices for the polynomials, $\ell_\theta$ and $q_\theta +\ell'\ell_\theta$ give rise to the same set of solutions, where $\ell'$ is an arbitrary linear polynomial. Similar multiplicative ambiguities also occur in the general case.
\end{Ex}

The correct algebraic structure to remove these ambiguities is the ideal in ring theory:
\begin{Def}
Let $R$ be a commutative ring, e.g.~the ring of polynomials in $D$ variables $R=\mathbb{C}[X_1,\dots, X_D]$ with addition and multiplication. An ideal of $R$ is a proper subset $\calI\varsubsetneq R$ such that:
\begin{enumerate}
\item{(i)} $\calI$ is additively closed,\\ i.e.~$f+g\in\calI$ for all $f,g\in\calI$
\item{(ii)} $\calI$ is closed under multiplication with $R$,\\ i.e. $f\cdot g\in\calI$ for all $f\in\calI$ and $g\in R$
\end{enumerate}
A radical ideal additionally fulfills:
\begin{enumerate}
\item{(iii)} $\calI$ is closed under taking roots,\\ i.e.~$f\in\calI$ if $f^n\in\calI$ for some $n\in\mathbb{N}$.
\end{enumerate}
\end{Def}
Hilbert's Nullstellensatz states that in the ring $\mathbb{C}[X_1,\dots, X_D]$, the radical ideals uniquely parameterize the different solution sets of polynomial equations. Thus we can remove the ambiguities in the parametric model by replacing sets of equations by ideals. Problem~\ref{Prob:eqsreg} then becomes
\begin{Prob}\label{Prob:eqsid}
Let $\mathcal{F}_\theta$ be a parametric family of radical ideals in $\mathbb{C}[X_1,\dots, X_D].$
Given input polynomials $q_1,\dots, q_n$, estimate a regression parameter $\theta$ such that $\mathcal{F}_\theta$ is ``close'' to the inputs $q_1,\dots, q_n.$
\end{Prob}
The radical ideals themselves are uniquely parameterized by parts of a certain manifold, the so-called Hilbert scheme; this automatically implies unique parametrization for the parametric family $\mathcal{F}_\theta$ if $\theta$ is a parameter of the Hilbert scheme. For example, the $d$-dimensional sub-vector spaces of $D$-space are equally parameterized by the possible row-spans of maximal rank $(d\times D)$-matrices. Algebraically, this corresponds to the non-singular part of the Grassmann manifold $\mbox{Gr} (d,D)$.

Also note that we in general cannot remove all the ambiguities in the input polynomials by putting them into a single ideal, since the measurements of the $q_i$ may be noisy, and the noise is on the coefficients of particular elements in the ideal. However, one could group for example some of them into classes of ideals, depending on the setting (for example the ideals of points in ordinary regression).

It remains to say what it means for input polynomials and regressor ideal $\mathcal{F}_\theta$ to be ``close''. As in ordinary regressions, there are different ways in which one can choose to penalize differences. For example, one can explicitly or numerically optimize a regularized loss function. On the other hand, a pragmatical approach is to measure the differences in terms of squared errors on the graded vector space structure of the ideal; e.g.~if the input polynomials are all degree $2$, one would sum the squared distances to the vector space consisting of degree two and less polynomials in $\mathcal{F}_\theta;$ or, if $\mathcal{F}_\theta$ is generated in degree $3$ and higher, then the least square error in the higher degree parts.

In the algorithm we present in section~\ref{sec:appidr}, we try to minimize squared errors in the graded parts, since quadratic optimization provides explicit and efficient solutions and thus deterministic algorithms.

\section{Algebraic Geometry of Genericity}
\label{app-generic}
In the paper, we have introduced the framework of ideal regression, where we estimate ideals from noisy input polynomials. In its algebraic formulation as Problem~\ref{Prob:eqsid}, we want to find a good regression parameter $\theta$ for the ideal $\mathcal{F}_\theta$. In ordinary regression, the generative assumption is, slightly reformulated, that the data points are points on the regression hyperplane, plus some independent noise, often even assumed i.i.d. Moreover, the sample points are assumed to be ``generic'' in the sense that the points are not the same and sufficiently distinct so that one can regress a hyperplane to them.

Thus, for ideal regression it is analogous and natural to postulate as generative model that the input polynomials are ``generic'' polynomials from the ideal $\mathcal{F}_\theta$ which are then disturbed by additional sampling noise. In the following section, we explain our probabilistic model for genericity, its relation to known types of genericity, and its theoretical implications for the ideal regression problem. The additional noise will be treated in the next section.

Since ideal regression is an algebraic procedure, knowledge about basic algebraic geometry will be required for an understanding of the following sections. In particular, the reader should be at least familiar with the following concepts before reading this section: polynomial rings, ideals, radicals, factor rings, algebraic sets, algebra-geometry correspondence (including Hilbert's Nullstellensatz), primary decomposition, height and dimension theory in rings. A good introduction into the necessary framework can be found in \cite{Cox}.

\label{sec:Alg-Generic}
\subsection{Definition of genericity}
\label{sec:gendef}
In the algebraic setting of the paper, we would like to calculate the radical and homogenous saturation of an ideal
$$\calI=\langle f_1,\dots, f_{n}\rangle.$$
This ideal $\calI$ is of a special kind: its generators $f_i$ are random, and are only subject to the constraints that they vanish on the linear subspace $S$ which we want to identify, and that they are homogenous of fixed degree. In order to derive meaningful results on how $\calI$ relates to $S$, or on the solvability of the problem, we need to model this kind of randomness.

In this section, we present a concept called genericity. Informally speaking, a generic situation is a situation without pathological degeneracies. In our case, it is reasonable to believe that apart from the conditions of homogeneity and the vanishing on $S$, there are no additional degeneracies in the choice of the generators. So, informally spoken, the ideal $\calI$ is generated by generic homogenous elements vanishing on $S.$ This section is devoted to developing a formal theory for addressing genericity, as it occurs for example in conditioned sampling as a generative assumption.

The concept of genericity is already widely used in theoretical computer science, combinatorics or discrete mathematics; there, it is however often defined inexactly or not at all, or it is only given as an ad-hoc definition for the particular problem. On the other hand, genericity is a classical concept in algebraic geometry, in particular in the theory of moduli. The interpretation of generic properties as probability-one-properties is also a known concept in applied algebraic geometry, e.g.~algebraic statistics. However, the application of probability distributions and genericity to the setting of generic ideals, in particular in the context of conditional probabilities, are original to the best of our knowledge, though not being the first one to involve generic resp.~general polynomials, see \cite{Iar84}. Generic polynomials and ideals have been also studied in \cite{Fro94}. A collection of results on generic polynomials and ideals which partly overlap with ours may also be found in the recent paper \cite{Par10}.

Before continuing to the definitions, let us explain what genericity should mean. Intuitively, generic objects are objects without unexpected pathologies or degeneracies. For example, if one studies say $n$ lines in the real plane, one wants to exclude pathological cases where lines lie on each other or where many lines intersect in one point. Having those cases excluded means examining the ``generic'' case, i.e. the case where there are $n(n+1)/2$ intersections, $n(n+1)$ line segments and so forth. Or when one has $n$ points in the plane, one wants to exclude the pathological cases where for example there are three affinely dependent points, or where there are more sophisticated algebraic dependencies between the points which one wants to exclude, depending on the problem.

In the points example, it is straightforward how one can define genericity in terms of sampling from a probability distribution:
one could draw the points under a suitable continuous probability distribution from real two-space. Then, saying that the points are ``generic'' just amounts to examine properties which are true with probability one for the $n$ points. Affine dependencies for example would then occur with probability zero and are automatically excluded from our interest. One can generalize this idea to the lines example: one can parameterize the lines by a parameter space, which in this case is two-dimensional (slope and ordinate), and then sample lines uniformly distributed in this space (one has of course to make clear what this means).  For example, lines lying on each other or more than two lines intersecting at a point would occur with probability zero, since the part of parameter space for this situation would have measure zero under the given probability distribution.

When we work with polynomials and ideals, the situation gets a bit more complicated, but the idea is the same. Polynomials are uniquely determined by their coefficients, so they can naturally be considered as objects in the vector space of their coefficients. Similarly, an ideal can be specified by giving the coefficients of some set of generators. Let us make this more explicit: suppose first we have given a single polynomial $f\in \C[X_1,\dots X_D]$ of degree $k$.

In multi-index notation, we can write this polynomial as a finite sum
$$f=\sum_{\alpha\in \NN^D}c_\alpha X^\alpha\,\quad \mbox{with}\;c_\alpha\in \C.$$
This means that the possible choices for $f$ can be parameterized by the ${D+k \choose k}$ coefficients $c_I$ with $\|I\|_1\le k.$ Thus polynomials of degree $k$ with complex coefficients can be parameterized by complex ${D+k \choose k}$-space.

Algebraic sets can be similarly parameterized by parameterizing the generators of the corresponding ideal. However, this correspondence is not one-to-one, as different generators may give rise to the same zero set. While the parameter space can be made unique by dividing out redundancies, which gives rise to the Hilbert scheme, we will instead use the redundant, though pragmatic characterization in terms of a finite dimensional vector space over $\C$ of the correct dimension.

We will now fix notation for the parameter space of polynomials and endow it with algebraic structure. The extension to ideals will then be derived later. Let us write $\calM_k$ for complex ${D+k \choose k}$-space (we assume $D$ as fixed), interpreting it as a parameter space for the polynomials of degree $k$ as shown above. Since the parameter space $\calM_k$ is isomorphic to complex ${D+k\choose k}$-space, we may speak about algebraic sets in $\calM_k$. Also, $\calM_k$ carries the complex topology induced by the topology on $\R^{2k}$ and by topological isomorphy the Lebesgue measure; thus it also makes sense to speak about probability distributions and random variables on $\calM_k.$ This dual interpretation will be the main ingredient in our definition of genericity, and will allow us to relate algebraic results on genericity to the probabilistic setting in the applications. As $\calM_k$ is a topological space, we may view any algebraic set in $\calM_k$ as an event if we randomly choose a polynomial in $\calM_k$:
\begin{Def}
Let $X$ be a random variable with values in $\calM_k$. Then an event for $X$ is called {\it algebraic event} or {\it algebraic property} if the corresponding event set in $\calM_k$ is an algebraic set. It is called {\it irreducible} if the corresponding event set in $\calM_k$ is an irreducible algebraic set.
\end{Def}

If an event $A$ is irreducible, this means that if we write $A$ as the event ``$A_1$ and $A_2$'', for algebraic events $A_1,A_2$, then $A=A_1$, or $A=A_2.$ We now give some examples for algebraic properties.

\begin{Ex}\rm\label{Ex:algevts}
The following events on $\calM_k$ are algebraic:
\begin{enumerate}
  \item The sure event.
  \item The empty event.
  \item The polynomial is of degree $n$ or less.
  \item The polynomial vanishes on a prescribed algebraic set.
  \item The polynomial is contained in a prescribed ideal.
  \item The polynomial is homogenous of degree $n$ or zero.
  \item The polynomial is homogenous.
  \item The polynomial is a square.
  \item The polynomial is reducible.
\end{enumerate}
Properties 1-6 are additionally irreducible.

We now show how to prove these claims: 1-2 are clear, we first prove that properties 3-6 are algebraic and irreducible. By definition, it suffices to prove that the subset of $\calM_k$ corresponding to those polynomials is an irreducible algebraic set. We claim: in any of those cases, the subset in question is moreover a linear subspace, and thus algebraic and irreducible. This can be easily verified by checking directly that if $f_1,f_2$ fulfill the property in question, then $f_1+\alpha f_2$ also fulfills the property.

Property 7 is algebraic, since it can be described as the disjunction of the properties ``The polynomial is homogenous of degree $n$ or zero'' for all $n\le k,$ for some fixed $k$. Those single properties can be described by linear subspaces of $\calM_k$ as above, thus property 7 is parameterized by the union of those linear subspaces. In general, these are not contained in each other, so property 6 is not irreducible.

Property 8 is algebraic, as we can check it through the vanishing of a system of generalized discriminant polynomials. One can show that it is also irreducible since the subset of $\calM_k$ in question corresponds to the image of a Veronese map (homogenization to degree $k$ is a strategy); however, since we will not need such a result, we do not prove it here.

Property 9 is algebraic, since factorization can also be checked by sets of equations. One has to be careful here though, since those equations depend on the degrees of the factors. For example, a polynomial of degree $4$ may factor into two polynomials of degree $1$ and $3$, or in two polynomials of degree $2$ each. Since in general each possible combination defines different sets of equations and thus different algebraic subsets of $\calM_k$, property 8 is in general not irreducible (for $k\le 3$ it is).
\end{Ex}

The idea defining a choice of polynomial as generic follows the intuition of the affirmed non-sequitur: a generic, resp.~generically chosen polynomial should not fulfill any algebraic property. A generic polynomial, having a particular simple (i.e.~irreducible) algebraic property, should not fulfill any other algebraic property which is not logically implied by the first one. Here, algebraic properties are regarded as the natural model for restrictive and degenerate conditions, while their logical negations are consequently interpreted as generic, as we have seen in Example~\ref{Ex:algevts}. These considerations naturally lead to the following definition of genericity in a probabilistic context:

\begin{Def}\label{Def:genrand}
Let $X$ be a random variable with values in $\calM_k$. Then $X$ is called {\it generic}, if for any irreducible algebraic events $A,B,$ the following holds:

The conditional probability $P_X(A|B)$ vanishes if and only if $B$ does not imply $A$.
\end{Def}
In particular, $B$ may also be the sure event.

Note that without giving a further explication, the conditional probability $P_X(A|B)$ is not well-defined, since we condition on the event $B$ which has probability zero. There is also no unique way of remedying this, as for example the Borel-Kolmogorov paradox shows. In section~\ref{sec:genalt}, we will discuss the technical notion which we adopt to ensure well-definedness.

Intuitively, our definition means that an event has probability zero to occur unless it is logically implied by the assumptions. That is, degenerate dependencies between events do not occur.

For example, non-degenerate multivariate Gaussian distributions or Gaussian mixture distributions on $\calM_k$ are generic distributions. More general, any positive continuous probability distribution which can be approximated by Gaussian mixtures is generic (see Example~\ref{Ex:gen-nongen}). Thus we argue that non-generic random variables are very pathological cases. Note however, that our intention is primarily not to analyze the behavior of particular fixed generic random variables (this is part of classical statistics). Instead, we want to infer statements which follow not from the particular structure of the probability function, but solely  from the fact that it is generic, as these statements are intrinsically implied by the conditional postulate in Definition~\ref{Def:genrand} alone. We will discuss the definition of genericity and its implications in more detail in section~\ref{sec:genalt}.

With this definition, we can introduce the terminology of a generic object: it is a generic random variable which is object-valued.

\begin{Def}
We call a generic random variable with values in $\calM_k$ a generic polynomial of degree $k.$ When the degree $k$ is arbitrary, but fixed (and still $\ge 1$), we will say that $f$ is a generic polynomial, or that $f$ is generic, if it is clear from the context that $f$ is a polynomial. If the degree $k$ is zero, we will analogously say that $f$ is a generic constant.\\

We call a set of constants or polynomials $f_1,\dots, f_m$ generic if they are generic and independent.\\

We call an ideal generic if it is generated by a set of $m$ generic polynomials.\\

We call an algebraic set generic if it is the vanishing set of a generic ideal.\\

Let $\calP$ be an algebraic property on a polynomial, a set of polynomials, an ideal, or an algebraic set (e.g.~homogenous, contained in an ideal et.). We will call a polynomial, a set of polynomials, or an ideal, a {\it generic} $\calP$ polynomial, set, or ideal, if it the conditional of a generic random variable with respect to $\calP$.\\

If $\calA$ is a statement about an object (polynomial, ideal etc), and $\calP$ an algebraic property, we will say briefly ``A generic $\calP$ object is $\calA$'' instead of saying ``A generic $\calP$ object is $\calA$ with probability one''.
\end{Def}

Note that formally, these objects are all polynomial, ideal, algebraic set etc -valued random variables. By convention, when we state something about a generic object, this will be an implicit probability-one statement. For example, when we say\\

``A generic green ideal is blue'',\\

this is an abbreviation for the by definition equivalent but more lengthy statement\\

``Let $f_1,\dots, f_m$ be independent generic random variables with values in $\calM_{k_1},\dots,\calM_{k_m}.$ If the ideal $\langle f_1,\dots, f_m\rangle$ is green, then with probability one, it is also blue - this statement is independent of the choice of the $k_i$ and the choice of which particular generic random variables we use to sample.\\

On the other hand, we will use the verb ``generic'' also as a qualifier for ``constituting generic distribution''. So for example, when we say\\

``The Z of a generic red polynomial is a generic yellow polynomial'',\\

this is an abbreviation of the statement\\

``Let $X$ be a generic random variable on $\calM_k,$ let $X'$ be the yellow conditional of $X$. Then the Z of $X'$ is the red conditional of some generic random variable - in particular this statement is independent of the choice of $k$ and the choice of $X$.''\\

It is important to note that the respective random variables will not be made explicit in the following subsections, since the statements will rely only on its property of being generic, and not on its particular structure which goes beyond being generic.\\

As an exemplary application of these concepts, we can formulate the noise-free version of the common marginals problem in terms of generic algebra:

\begin{Prob}\label{Prob:SSA-alg}
Let $\fraks=\Id(S)$, where $S$ is an unknown $d$-dimensional subspace of $\C^D$. Let
$$\calI=\langle f_1,\dots, f_m \rangle$$
with $f_i\in \fraks$ generic of fixed degree each (in our case, one and two), such that $\sqrt{\calI}=\fraks.$

Then determine a reduced H-basis (or another simple generating system) for $\fraks.$
\end{Prob}
We will derive a noisy version for the more general setting of ideal regression in section~\ref{sec:appidr}.

\subsection{Zero-measure conditionals, and relation to other types of genericity}\label{sec:genalt}
In this section, se will discuss the definition of genericity in Definition~\ref{Def:genrand} and ensure its well-definedness. Then we will invoke alternative definitions for genericity and show their relation to our probabilistic intuitive approach from section~\ref{sec:gendef}. As this section contains technical details and is not necessary for understanding the rest of the appendix, the reader may opt to skip it.

An important concept in our definition of genericity in Definition~\ref{Def:genrand} is the conditional probability $P_X(A|B)$. As $B$ is an algebraic set, its probability $P_X(B)$ is zero, so the Bayesian definition of conditional cannot apply. There are several ways to make it well-defined; in the following, we explain the Definition of conditional we use in Definition~\ref{Def:genrand}. The definition of conditional we use is one which is also often applied in this context.
\begin{Rem}\em\label{Rem:genmeas}
Let $X$ be a real random variable (e.g.~with values in $\calM_k$) with probability measure $\mu$. If $\mu$ is absolutely continuous, then by the theorem of Radon-Nikodym, there is a unique continuous density $p$ such that
$$\mu(U)=\int_U p\, d\lambda$$
for any Borel-measurable set $U$ and the Lebesgue measure $\lambda$. If we assume that $p$ is a continuous function, it is unique, so we may define a restricted measure $\mu_B$ on the event set of $B$ by setting
$$\nu(U)=\int_U p\, dH,$$
for Borel subsets of $U$ and the Hausdorff measure $H$ on $B$. If $\nu(B)$ is finite and non-zero, i.e.~$\nu$ is absolutely continuous with respect to $H$, then it can be renormalized to yield a conditional probability measure $\mu(.)|_B=\nu(.)/\nu(B).$ The conditional probability $P_X(A|B)$ has then to be understood as
$$P_X(A|B)=\int_{B}\mathbbm{1} (A\cap B)\,d\mu\mid_B,$$
whose existence in particular implies that the Lebesgue integrals $\nu (B)$ are all finite and non-zero.
\end{Rem}

As stated, we adopt this as the definition of conditional probability for algebraic sets $A$ and $B$. It is important to note that we have made implicit assumptions on the random variable $X$ by using the conditionals $P_X(A|B)$ in Remark~\ref{Rem:genmeas} (and especially by assuming that they exist): namely, the existence of a continuous density function and existence, finiteness, and non-vanishing of the Lebesgue integrals. Similarly, by stating Definition~\ref{Def:genrand} for genericity, we have made similar assumptions on the generic random variable $X$, which can be summarized as follows:

\begin{Ass}\label{Ass:gen}
$X$ is an absolutely continuous random variable with continuous density function $p$, and for every algebraic event $B$, the Lebesgue integrals
$$\int_B p\, dH,$$
where $H$ is the Hausdorff measure on $B$, are non-zero and finite.
\end{Ass}

This assumption implies the existence of all conditional probabilities $P_X(A|B)$ in Definition~\ref{Def:genrand}, and are also necessary in the sense that they are needed for the conditionals to be well-defined. On the other hand, if those assumptions are fulfilled for a random variable, it is  automatically generic:

\begin{Rem}\em\label{Prop:gen=cont}
Let $X$ be a $\calM_k$-valued random variable, fulfilling the Assumptions in~\ref{Ass:gen}. Then, the probability density function of $X$ is strictly positive. Moreover, $X$ is a generic random variable.
\end{Rem}
\begin{proof}\em
Let $X$ be a $\calM_k$-valued random variable fulfilling the Assumptions in~\ref{Ass:gen}. Let $p$ be its continuous probability density function.

We first show positivity: If $X$ would not be strictly positive, then $p$ would have a zero, say $x$. Taking $B=\{x\},$ the integral $\int_B p\, dH$ vanishes, contradicting the assumption.

Now we prove genericity, i.e.~that for arbitrary irreducible algebraic properties $A,B$ such that $B$ does not imply $A$, the conditional probability $P_X(A|B)$ vanishes. Since $B$ does not imply $A$, the algebraic set defined by $B$ is not contained in $A$. Moreover, as $B$ and $A$ are irreducible and algebraic, $A\cap B$ is also of positive codimension in $B$. Now by assumption, $X$ has a positive continuous probability density function $f$ which by assumption restricts to a probability density on $B$, being also positive and continuous. Thus the integral
$$P_X(A|B)=\int_B \mathbbm{1}_A f(x)\, dH,$$
where $H$ is the Hausdorff measure on $B$, exists. Moreover, it is zero, as we have derived that $A$ has positive codimension in $B$.
\end{proof}

This means that already under mild assumptions, which merely ensure well-definedness of the statement in the Definition~\ref{Def:genrand} of genericity, random variables are generic. The strongest of the comparably mild assumptions are the convergence of the conditional integrals, which allow us to renormalize the conditionals for all algebraic events. In the following example, a generic and a non-generic probability distribution are presented.

\begin{Ex}\rm\label{Ex:gen-nongen}
Gaussian distributions and Gaussian mixture distributions are generic, since for any algebraic set $B$, we have
$$\int_B \mathbbm{1}_{\mathcal{B}(t)}\, dH = O(t^{\dim B}),$$
where $\mathcal{B}(t)=\{x\in\mathbb{R}^n\;;\; \|x\|<t\}$ is the open disc with radius $t.$ Note that this particular bound is false in general and may grow arbitrarily large when we omit $B$ being algebraic, even if $B$ is a smooth manifold.
Thus $P_X(A|B)$ is bounded from above by an integral (or a sum) of the type
$$\int_{0}^\infty\exp(-t^2)t^a\;dt\quad\mbox{with}\; a\in\mathbb{N}$$
which is known to be finite.

Furthermore, sums of generic distributions are again generic; also, one can infer that any continuous probability density dominated by the distribution of a generic density defines again a generic distribution.

An example of a non-generic but smooth distribution is given by the density function
$$p(x,y)=\frac{1}{\mathcal{N}}e^{-x^4y^4}$$
where $\mathcal{N}$ is some normalizing factor. While $p$ is integrable on $\mathbb{R}^2,$ its restriction to the coordinate axes $x=0$ and $y=0$ is constant and thus not integrable.
\end{Ex}

Now we will examine different known concepts of genericity and relate them briefly to the one we have adopted.

A definition of genericity in combinatorics and geometry which can be encountered in different variations is that there exist no degenerate interpolating functions between the objects:

\begin{Def}\label{Def:gencomb}
Let $P_1,\dots, P_m$ be points in the vector space $\mathbb{C}^n$. Then $P_1,\dots, P_m$ are general position (or generic, general) if no $n+1$ points lie on a hyperplane. Or, in a stronger version: for any $d\in\mathbb{N}$, no (possibly inhomogenous) polynomial of degree $d$ vanishes on ${n+d \choose d}+1$ different $P_i$.
\end{Def}
As $\calM_k$ is a finite dimensional $\mathbb{C}$-vector space, this definition is in principle applicable to our situation. However, this definition is deterministic, as the $P_i$ are fixed and no random variables, and thus preferable when making deterministic statements. Note that the stronger definition is equivalent to postulating general position for the points $P_1,\dots, P_m$ in any polynomial kernel feature space.

Since not lying on a hyperplane (or on a hypersurface of degree $d$) in $\mathbb{C}^n$ is a non-trivial algebraic property for any point which is added beyond the $n$-th (resp.~the ${n+d\choose d}$-th) point $P_i$ (interpreted as polynomial in $\calM_k$), our definition of genericity implies general position. This means that generic polynomials $f_1,\dots, f_m\in\calM_k$ (almost surely) have the deterministic property of being in general position as stated in Definition~\ref{Def:gencomb}. A converse is not true for two reasons: first, the $P_i$ are fixed and no random variables. Second, even if one would define genericity in terms of random variables such that the hyperplane (resp.~hypersurface) conditions are never fulfilled, there are no statements made on conditionals or algebraic properties other than containment in a hyperplane, also Lebesgue zero sets are not excluded from occuring with positive probability.

Another example where genericity classically occurs is algebraic geometry, where it is defined rather general for moduli spaces. While the exact definition may depend on the situation or the particular moduli space in question, and is also not completely consistent, in most cases, genericity is defined as follows: general, or generic, properties are properties which hold on a Zariski-open subset of an (irreducible) variety, while very generic properties hold on a countable intersection of Zariski-open subsets  (which are thus paradoxically ''less'' generic than general resp.~generic properties in the algebraic sense, as any general resp.~generic property is very generic, but the converse is not necessarily true). In our special situation, which is the affine parameter space of tuples of polynomials, these definitions can be rephrased as follows:
\begin{Def}\label{Def:gencomb}
Let $B\subseteq\mathbb{C}^k$ be an irreducible algebraic set, let $P=(f_1,\dots, f_m)$ be a tuple of polynomials, viewed as a point in the parameter space $B.$ Then a statement resp.~property $A$ of $P$ is called very generic if it holds on the complement of some countable union of algebraic sets in $B.$ A statement resp.~property $A$ of $P$ is called general (or generic) if it holds on the complement of some finite union of algebraic sets in $B.$
\end{Def}
This definition is more or less equivalent to our own; however, our definition adds the practical interpretation of generic/very generic/general properties being true with probability one, while their negations are subsequently true with probability zero. In more detail, the correspondence is as follows:
If we restrict ourselves only to algebraic properties $A$, it is equivalent to say that the property $A$ is very generic, or general for the $P$ in $B$, and to say with our original definition that a generic $P$ fulfilling $B$ is also $A$; since if $A$ is by assumption an algebraic property, it is both an algebraic set and a complement of a finite (countable) union of algebraic sets in an irreducible algebraic set, so $A$ must be equal to an irreducible component of $B$; since $B$ is irreducible, this implies equality of $A$ and $B$. On the other hand, if $A$ is an algebraic property, it is equivalent to say that the property not-$A$ is very generic, or general for the $P$ in $B$, and to say with our original definition that a generic $P$ fulfilling $B$ is not $A$ - this corresponds intuitively to the probability-zero condition $P(A|B)=0$ which states that non-generic cases do not occur. Note that by assumption, not-$A$ is then always the complement of a finite union of algebraic sets.

\subsection{Arithmetic of generic polynomials}
In this subsection, we study how generic polynomials behave under classical operations in rings and ideals. This will become important later when we study generic polynomials and ideals.\\

To introduce the reader to our notation of genericity, and since we will use the presented facts and similar notations implicitly later, we prove the following:

\begin{Lem}\label{Lem:Gen-arith}
Let $f\in\C [X_1,\dots, X_D]$ be generic of degrees $k.$ Then:\\
\itboxx{i} The product $\alpha f$ is generic of degree $k$ for any fixed $\alpha\in\C\setminus \{\0\}.$\\
\itboxx{ii} The sum $f + g$ is generic of degree $k$ for any $g\in \C [X_1,\dots, X_D]$ of degree $k$ or smaller.\\
\itboxx{iii} The sum $f + g$ is generic of degree $k$ for any generic $g\in \C [X_1,\dots, X_D]$ of degree $k$ or smaller.
\end{Lem}
\begin{proof}\em
(i) is clear since the coefficients of $g_1$ are multiplied only by a constant. (ii) follows directly from the definitions since adding a constant $g$ only shifts the coefficients without changing genericity. (iii) follows since $f,g$ are independently sampled: if there were algebraic dependencies between the coefficients of $f+g$, then either $f$ or $g$ was not generic, or the $f,g$ are not independent, which both would be a contradiction to the assumption.
\end{proof}

Recall again what this Lemma means: for example, Lemma~\ref{Lem:Gen-arith} (i) does not say, as one could think:\\

``Let $X$ be a generic random variable with values in the vector space of degree $k$ polynomials. Then $X=\alpha X$ for any $\alpha\in \C\setminus \{0\}.$''\\

The correct translation of Lemma~\ref{Lem:Gen-arith} (i) is:\\

``Let $X$ be a generic random variable with values in the vector space of degree $k$ polynomials. Then $X'=\alpha X$ for any fixed $\alpha\in \C\setminus \{0\}$ is a generic random variable with values in the vector space of degree $k$ polynomials''\\

The other statements in Lemma~\ref{Lem:Gen-arith} have to be interpreted similarly.\\

The following remark states how genericity translates through dehomogenization:
\begin{Lem}\label{Lem:dehom}
Let $f\in\C [X_1,\dots, X_D]$ be a generic homogenous polynomial of degree $d.$ \\
Then the dehomogenization $f(X_1,\dots, X_{D-1},1)$ is a generic polynomial of degree $d$ in the polynomial ring $\C [X_1,\dots, X_{D-1}].$\\

Similarly, let $\fraks\subseteq \C [X_1,\dots, X_D]$ be a generic homogenous ideal. Let $f\in \fraks$ be a generic homogenous polynomial of degree $d.$ \\
Then the dehomogenization $f(X_1,\dots, X_{D-1},1)$ is a generic polynomial of degree $d$ in the dehomogenization of $\fraks.$
\end{Lem}
\begin{proof}\em
For the first statement, it suffices to note that the coefficients of a homogenous polynomial of degree $d$ in the variables $X_1,\dots, X_D$ are in bijection with the coefficients of a polynomial of degree $d$ in the variables $X_1,\dots, X_{D-1}$ by dehomogenization. For the second part, recall that the dehomogenization of $\fraks$ consists exactly of the dehomogenizations of elements in $\fraks.$ In particular, note that the homogenous elements of $\fraks$ of degree $d$ are in bijection to the elements of degree $d$ in the dehomogenization of $\fraks$. The claims then follows from the definition of genericity.
\end{proof}

\subsection{Dimension of generic spans and ideals}
In this subsection, we will derive the first results on generic ideals. We will derive an statement about spans of generic polynomials, and generic versions of Krull's principal ideal and height theorems which will be the main tool in controlling the structure of generic ideals. This has immediate applications for the cumulant comparison problem.\\

We begin with a probably commonly known result, formulated in terms of genericity:

\begin{Prop}\label{Prop:GenVec}
Let $P$ be an algebraic property such that the polynomials with property $P$ form a vector space $V$. Let $f_1,\dots, f_m\in \C[X_1,\dots X_D]$ be generic polynomials satisfying $P.$ Then
$$\rk \lspan (f_1,\dots, f_m)=\min (m, \dim V).$$
\end{Prop}
\begin{proof}\em
It suffices to prove: if $i\le M,$ then $f_i$ is linearly independent from $f_1,\dots f_{i-1}$ with probability one. Assuming the contrary would mean that for some $i$, we have
$$f_i=\sum_{k=0}^{i-1}f_k c_k\quad\mbox{for some}\; c_k\in \CC,$$
thus giving several equations on the coefficients of $f_i.$ But these are fulfilled with probability zero by the genericity assumption, so the claim follows.
\end{proof}

This may be seen as a straightforward generalization of the statement: the span of $n$ generic points in $\C^D$ has dimension $\min (n,D).$\\

We now proceed to another nontrivial result which will now allow us to formulate a generic version of Krull's principal ideal theorem:

\begin{Prop}\label{Prop:NoZero}
Let $Z\subseteq \C^D$ be a non-empty algebraic set, let $f\in \C[X_1,\dots X_D]$ generic. Then $f$ is a non-zero divisor in $\calO(Z)=\C[X_1,\dots X_D]/\Id(Z).$
\end{Prop}
\begin{proof}\em
We claim: being a zero divisor in $\calO(Z)$ is an irreducible algebraic property. We will prove that the zero divisors in $\calO(Z)$ form a linear subspace of $\calM_k,$ and linear spaces are irreducible.\\

For this, one checks that sums and scalar multiples of zero divisors are also zero divisors: if $g_1,g_2$ are zero divisors, there must exist $h_1,h_2$ such that $g_1h_1=g_2h_2=0.$ Now for any $\alpha\in \C,$ we have that
$$(g_1+\alpha g_2) (h_1h_2)=(g_1h_1)h_2 + (g_2h_2)\alpha h_1= 0.$$
This proves that $(g_1+\alpha g_2)$ is also a zero divisor, proving that the zero divisors form a linear subspace and thus an irreducible algebraic property.

To apply the genericity assumption to argue that this event occurs with probability zero, we must exclude the possibility that being a zero divisor is trivial, i.e.~always the case. This is equivalent to proving that the linear subspace has positive codimension, which is true if and only if there exists a non-zero divisor in $\calO(Z).$ But a non-zero divisor always exists since we have assumed $Z$ is non-empty: thus $\Id(Z)$ is a proper ideal, and $\calO(Z)$ contains $\C,$ which contains a non-zero divisor, e.g.~the one.\\

So by the genericity assumption, the event that $f$ is a zero divisor occurs with probability zero, i.e.~a generic $f$ is not a zero divisor. Note that this does not depend on the degree of $f.$
\end{proof}
This result is already known, compare Conjecture B in \cite{Par10}.

A straightforward generalization using the same proof technique is given by the following
\begin{Cor}\label{Cor:NoZero}
Let $\calI\subseteq \C[X_1,\dots, X_D]$, let $P$ be a non-trivial algebraic property. Let $f\in \C[X_1,\dots X_D]$ be a generic polynomial with property $P$. If one can write $f=f'+c$, where $f'$ is a generic polynomial subject to some property $P'$, and $c$ is a generic constant, then $f$ is non-zero divisor in $\C[X_1,\dots, X_D]/\calI.$
\end{Cor}
\begin{proof}\em
First note that $f$ is a zero divisor in $\C[X_1,\dots, X_D]/\calI$ if and only if $f$ is a zero divisor in $\C[X_1,\dots, X_D]/\sqrt{\calI}.$ This allows us to reduce to the case that $\calI=\Id (Z)$ for some algebraic set $Z\subseteq \C^D.$\\

Now, as in the proof of Proposition~\ref{Prop:NoZero}, we see that being a zero divisor in $\calO(Z)$ is an irreducible algebraic property and corresponds to a linear subspace of $\calM_k$, where $k=\deg f.$ The zero divisors with property $P$ are thus contained in this linear subspace. Now let $f$ be generic with property $P$ as above. By assumption, we may write $f=f'+c.$ But $c$ is (generically) a non-zero divisor, so $f$ is also not a zero divisor, since the zero divisors form a linear subspace of $\calM_k.$ Thus $f$ is non-zero divisor. This proves the claim.
\end{proof}

Note that Proposition~\ref{Prop:NoZero} is actually a special case of Corollary~\ref{Cor:NoZero}, since we can write any generic polynomial $f$ as $f'+c$, where $f'$ is generic of the same degree, and $c$ is a generic constant.\\

The major tool to deal with the dimension of generic intersections is Krull's principal ideal theorem:
\begin{Thm}[Krull's principal ideal theorem]\label{Thm:KrullPI}
Let $R$ be a commutative ring with unit, let $f\in R$ be non-zero and non-invertible. Then
$$\htid \langle f\rangle\le 1,$$
with equality if and only if $f$ is not a zero divisor in $R$.
\end{Thm}
The reader unfamiliar with height theory may take
$$\htid \calI = \codim \VS(\calI)$$
as the definition for the height of an ideal (cave: codimension has to be taken in $R$).\\

Reformulated geometrically for our situation, Krull's principal ideal theorem implies:
\begin{Cor}\label{Cor:KrullPI-geom}
Let $Z$ be a non-empty algebraic set in $\C^D.$Then
$$\codim (Z\cap \VS(f))\le \codim Z+1.$$
\end{Cor}
\begin{proof}\em
Apply Krull's principal ideal theorem to the ring $R=\calO(Z)=\C [X_1,\dots, X_D]/\Id(Z).$
\end{proof}
Together with Proposition~\ref{Prop:NoZero}, one gets a generic version of Krull's principal ideal theorem:

\begin{Thm}[Generic principal ideal theorem]\label{Thm:KrullGenPI}
Let $Z$ be a non-empty algebraic set in $\C^D$, let $R=\calO (Z),$  and let $f\in \C[X_1,\dots, X_D]$ be generic. Then we have
$$\htid \langle f\rangle = 1.$$
\end{Thm}
In its geometric formulation, we obtain the following result.
\begin{Cor}\label{Cor:KrullPI}
Consider an algebraic set $Z\subseteq \C^D,$ and the algebraic set $\VS(f)$ for some generic $f\in \C [X_1,\dots, X_D].$
Then
$$\codim (Z\cap  \VS (f))=\min (\codim Z + 1,\; D+1).$$
\end{Cor}
\begin{proof}\em
This is just a direct reformulation of Theorem~\ref{Thm:KrullGenPI} in the vein of Corollary~\ref{Cor:KrullPI-geom}. The only additional thing that has to be checked is the case where $\codim Z = D+1,$ which means that $Z$ is the empty set. In this case, the equality is straightforward.
\end{proof}

The generic version of the principal ideal theorem straightforwardly generalizes to a generic version of Krull's height theorem. We first mention the original version:
\begin{Thm}[Krull's height theorem]\label{Thm:KrullHt}
Let $R$ be a commutative ring with unit, let $\calI=\langle f_1,\dots, f_m \rangle \subseteq R$ be an ideal. Then
$$\htid \calI \le m,$$
with equality if and only if $f_1,\dots, f_m$ is an $R$-regular sequence, i.e.~$f_i$ is not invertible and not a zero divisor in the ring $R/\langle f_1,\dots, f_{i-1}\rangle$ for all $i$.
\end{Thm}
The generic version can be derived directly from the generic principal ideal theorem:

\begin{Thm}[Generic height theorem]\label{Thm:KrullGenHt}
Let $Z$ be an algebraic set in $\C^D,$ let $\calI=\langle f_1,\dots, f_m\rangle$ be a generic ideal in $\C [X_1,\dots, X_D].$ Then
$$\htid (\Id(Z)+\calI) = \min (\codim Z + m,\; D+1).$$
\end{Thm}
\begin{proof}\em
We will write $R=\calO(Z)$ for abbreviation.

First assume $m\le D+1-\codim Z.$ It suffices to show that $f_1,\dots, f_m$ forms an $R$-regular sequence, then apply Krull's height theorem. In Proposition~\ref{Prop:NoZero}, we have proved that $f_i$ is not a zero divisor in the ring
$\calO(Z\cap\VS(f_1,\dots, f_{i-1}))$ (note that the latter ring is nonzero by Krull's height theorem). By Hilbert's Nullstellensatz, this is the same as the ring $R/\sqrt{\langle f_1,\dots, f_{i-1}\rangle}.$ But by the definition of radical, this implies that $f_i$ is a non-zero divisor in the ring $R/\langle f_1,\dots, f_{i-1}\rangle,$ since if $f_i\cdot h=0$ in the first ring, we have
$$(f_i\cdot h)^N=f_i\cdot (f_i^{N-1}h^N)=0$$
in the second. Thus the $f_i$ form an $R$-regular sequence, proving the theorem for the case $m\le D+1-\codim Z.$

If now $m> k:=D+1-\codim Z,$ the above reasoning shows that the radical of $\Id(Z)+\langle f_1,\dots, f_k\rangle$ is the module $\langle 1\rangle,$ which means that those are equal. Thus
$$\Id(Z)+\langle f_1,\dots, f_k\rangle=\Id(Z)+\langle f_1,\dots, f_m\rangle=\langle 1\rangle,$$
proving the theorem.

Note that we could have proved the generic height theorem also directly from the generic principal ideal theorem by induction.
\end{proof}

Again, we give the geometric interpretation of Krull's height theorem:

\begin{Cor}\label{Cor:genint}
Let $Z_1$ be an algebraic set in $\C^D$, let $Z_2$ be a generic algebraic set in $\C^D$. Then one has
\begin{align*}
\codim (Z_1\cap Z_2)=\min (\codim Z_1+\codim Z_2,\; D+1).
\end{align*}
\end{Cor}
\begin{proof}\em
This follows directly from two applications of the generic height theorem~\ref{Thm:KrullGenHt}: first for $Z=\C^D$ and $Z_2=\VS(\calI)$, showing that $\codim Z_2$ is equal to the number $m$ of generators of $\calI;$ then, for $Z=Z_1$ and $Z_2=\VS(\calI),$ and substituting $m=\codim Z_2.$
\end{proof}

We can now immediately formulate a homogenous version of Proposition~\ref{Cor:genint}:

\begin{Cor}\label{Cor:genintproj}
Let $Z_1$ be a homogenous algebraic set in $\C^D$, let $Z_2$ be a generic homogenous algebraic set in $\C^D$. Then one has
$$\codim (Z_1\cap Z_2)=\min (\codim Z_1+\codim Z_2,\; D).$$
\end{Cor}
\begin{proof}\em
Note that homogenization and dehomogenization of a non-empty algebraic set do not change its codimension, and homogenous algebraic sets always contain the origin. Also, one has to note that by Lemma~\ref{Lem:dehom}, the dehomogenization of $Z_2$ is a generic algebraic set in $\C^{D-1}.$
\end{proof}

Finally, using Corollary~\ref{Cor:NoZero}, we want to give a more technical variant of the generic height theorem, which will be of use in later proofs. First, we introduce some abbreviating notations:
\begin{Def}
Let $f\in \C[X_1,\dots X_D]$ be a generic polynomial with property $P$. If one can write $f=f'+c$, where $f'$ is a generic polynomial subject to some property $P'$, and $c$ is a generic constant, we say that $f$ has {\it independent constant term}. If $c$ is generic and independent with respect to some collection of generic objects, we say that $f$ has independent constant term with respect to that collection.
\end{Def}
In this terminology, Corollary~\ref{Cor:NoZero} rephrases as: a generic polynomial with independent constant term is a non-zero divisor. Using this, we can now formulate the corresponding variant of the generic height theorem:

\begin{Lem}\label{Lem:KrullGenHt}
Let $Z$ be an algebraic set in $\C^D.$ Let $f_1,\dots, f_m\in\C[X_1,\dots, X_D]$ be generic, possibly subject to some algebraic properties, such that $f_i$ has independent constant term with respect to $Z$ and $f_1,\dots, f_{i-1}.$ Then
$$\htid (\Id(Z)+\calI) = \min (\codim Z + m,\; D+1).$$
\end{Lem}
\begin{proof}\em
Using Corollary~\ref{Cor:NoZero}, one obtains that $f_i$ is non-zero divisor modulo $\Id(Z)+\langle f_1,\dots, f_{i+1}\rangle.$ Using Krull's height theorem yields the claim.
\end{proof}

\subsection{Dimension of conditioned generic ideals}
The generic height theorem~\ref{Thm:KrullGenHt} has allowed us to make statements about the structure of ideals generated by generic elements without constraints. However, the ideal $\calI$ in our the cumulant comparison problem is generic subject to constraints: namely, its generators are contained in a prescribed ideal, and they are homogenous. In this subsection, we will use the theory developed so far to study generic ideals and generic ideals subject to some algebraic properties, e.g.~generic ideals contained in other ideals. We will use these results to derive an identifiability result on the marginalization problem which has been derived already less rigourously in the supplementary material of \cite{PRL:SSA:2009} for the special case of Stationary Subspace Analysis.

\begin{Prop}\label{Prop:dehom-rad-generic}
Let $\fraks \subseteq \C [X_1,\dots, X_D]$ be an ideal, having an H-basis $g_1,\dots, g_n$. Let
$$\calI=\langle f_1,\dots, f_m\rangle,\quad m\ge \max(D+1, n)$$
with generic $f_i\in \fraks$ such that
$$\deg f_i\ge \max_j \left(\deg g_j\right)\quad \mbox{for all}\; 1\le i\le m.$$
Then $\calI=\fraks.$
\end{Prop}
\begin{proof}\em
First note that since the $g_i$ form a degree-first Groebner basis, a generic $f\in \fraks$ is of the form
$$f=\sum_{k=1}^n g_kh_k\quad\mbox{with generic}\;h_k,$$
where the degrees of the $h_k$ are appropriately chosen, i.e. $\deg h_k\le \deg f - \deg g_k$. \\

So we may write
$$f_i=\sum_{k=1}^n g_kh_{ki}\quad\mbox{with generic}\;h_{ki},$$
where the $h_{ki}$ are generic with appropriate degrees, and independently chosen. We may also assume that the $f_i$ are ordered increasingly by degree.\\

To prove the statement, it suffices to show that $g_j\in \calI$ for all $j$. Now the height theorem~\ref{Thm:KrullGenHt} implies that
$$\langle h_{11},\dots h_{1m}\rangle=\langle 1\rangle,$$
since the $h_{ki}$ were independently generic, and $m\ge D+1.$ In particular, there exist polynomials $s_1,\dots, s_m$ such that
$$\sum_{i=1}^m s_i h_{1i}=1.$$
Thus we have that
\begin{align*}
\sum_{i=1}^m s_i f_i = \sum_{i=1}^m s_i \sum_{k=1}^n g_kh_{ki}= \sum_{k=1}^n g_k \sum_{i=1}^m s_ih_{ki}\\
=g_1+ \sum_{k=2}^n g_k \sum_{i=1}^m s_ih_{ki}=:g_1+ \sum_{k=2}^n g_k h'_k.
\end{align*}
Subtracting a suitable multiple of this element from the $f_1,\dots, f_m,$ we obtain
$$f'_i=\sum_{k=2}^n g_k(h_{ki}-h_{1i}h'_k)=:\sum_{k=2}^n g_k h'_{ki}.$$
We may now consider $h_{1i}h'_k$ as fixed, while the $h_{ki}$ are generic. In particular, the $h'_{ki}$ have independent constant term, and using Lemma~\ref{Lem:KrullGenHt}, we may conclude that
$$\langle h'_{21},\dots, h'_{2m} \rangle=\langle 1\rangle,$$
allowing us to find an element of the form
$$g_2+\sum_{k=3}^n g_k \cdot\dots$$
in $\calI$. Iterating this strategy by repeatedly applying Lemma~\ref{Lem:KrullGenHt}, we see that $g_k$ is contained in $\calI,$ because the ideals $\calI$ and $\fraks$ have same height. Since the numbering for the $g_j$ was arbitrary, we have proved that $g_j\in \calI$, and thus the proposition.
\end{proof}
The following example shows that in general, we may not take the degrees of the $f_i$ lower than the maximal degree of the $g_j$ in the proposition, i.e.~the condition on the degrees is necessary:
\begin{Ex}\rm
Keep the notations of Proposition~\ref{Prop:dehom-rad-generic}. Let $\fraks=\langle X_2-X_1^2, X_3\rangle,$ and $f_i\in \fraks$ generic of degree one. Then
$$\langle f_1,\dots, f_m\rangle = \langle X_3\rangle.$$
This example can be generalized to yield arbitrarily bad results if the condition on the degrees is not fulfilled.

However note that when $\fraks$ is generated by linear forms, as in the marginalization problem, the condition on the degrees vanishes.
\end{Ex}

We may use Proposition~\ref{Prop:dehom-rad-generic} also in another way to derive a more detailed version of the generic height theorem for constrained ideals:
\begin{Prop}\label{Prop:KrullHt-algset}
Let $V$ be a fixed $d$-codimensional algebraic set in $\C^D.$ Assume that there exist $d$ generators $g_1,\dots, g_d$ for $\Id(V).$
Let $f_1,\dots, f_m$ be generic forms in $\Id (V)$ such that $\deg f_i\ge \max_j \left(\deg g_j\right)$. Then we can write $\VS (f_1,\dots, f_m)=V\cup U$ with $U$ an algebraic set of
$$\codim U\ge\min (m,\;D+1),$$
the equality being strict for $m < \codim V.$
\end{Prop}
\begin{proof}\em
If $m\ge D+1$, this is just a direct consequence of Proposition~\ref{Prop:dehom-rad-generic}.\\

First assume $m = d.$ Consider the image of the situation modulo $X_{m},\dots, X_D.$ This corresponds to looking at the situation
$$\VS (f_1,\dots, f_m)\cap H\subseteq H\cong \C^{m-1},$$
where $H$ is the linear subspace given by $X_m=\dots = X_D=0.$ Since the coordinate system was generic, the images of the $f_i$ will be generic, and we have by Proposition~\ref{Prop:dehom-rad-generic} that $\VS (f_1,\dots, f_m)\cap H = V\cap H.$ Also, the $H$ can be regarded as a generic linear subspace, thus by Corollary~\ref{Cor:genint}, we see that $\VS (f_1,\dots, f_m)$ consists of $V$ and possibly components of equal or higher codimension. This proves the claim for $m = \codim V.$

Now we prove the case $m\ge d.$ We may assume that $m=D+1$ and then prove the statement for the sets
$\VS (f_1,\dots, f_i), d\le i\le m.$ By the Lasker-Noether-Theorem, we may write
$$\VS (f_1,\dots, f_d)= V \cup Z_1 \cup\dots \cup Z_N$$
for finitely many irreducible components $Z_j$ with $\codim Z_j\ge \codim V.$ Proposition~\ref{Prop:dehom-rad-generic} now states that $$\VS (f_1,\dots, f_m)=V.$$
For $i\ge d,$ write now
$$Z_{ji}=Z_j\cap \VS (f_1,\dots, f_i)= Z_j\cap \VS (f_{d+1},\dots, f_i).$$
With this, we have the equalities
\begin{align*}
\VS (f_1,\dots, f_i)&= \VS (f_1,\dots, f_d)\cap \VS (f_{d+1},\dots, f_i)\\
&= V \cup (Z_1\cap \VS (f_{d+1},\dots, f_i))\cup\dots \\
& \phantom{V \cup (Z_1\cap \VS}\cup (Z_N\cap \VS (f_{d+1},\dots, f_i))\\
&= V\cup Z_{1i}\cup\dots\cup Z_{Ni}.
\end{align*}
for $i\ge d.$ Thus, reformulated, Proposition~\ref{Prop:dehom-rad-generic} states that $Z_{jm}=\varnothing$ for any $j$. We can now infer by Krull's principal ideal theorem~\ref{Thm:KrullPI} that
$$\codim Z_{ji}\le \codim Z_{j,i-1}+1$$
for any $i,j$. But since $\codim Z_{jm}=D+1,$ and $\codim Z_{jd}\ge d,$ we thus may infer that $\codim Z_{ji}\ge i$ for any $d\le i\le m.$ Thus we may write
$$\VS (f_1,\dots, f_i)=V\cup U\quad\mbox{with}\;U=Z_{1i}\cup\dots\cup Z_{Ni}$$
with $\codim U\ge i,$ which proves the claim for $m\ge \codim V.$

The case $m < \codim V$ can be proved again similarly by Krull's principal ideal theorem~\ref{Thm:KrullPI}: it states that the codimension of $\VS (f_1,\dots, f_i)$ increases at most by one with each $i$, and we have seen above that it is equal to $\codim V$ for $i=\codim V.$ Thus the codimension of $\VS (f_1,\dots, f_i)$ must have been $i$ for every $i\le \codim V.$ This yields the claim.
\end{proof}
Note that depending on $V$ and the degrees of the $f_i,$ it may happen that even in the generic case, the equality in Proposition~\ref{Prop:KrullHt-algset} is not strict for $m\ge \codim V$:
\begin{Ex}\rm
Let $V$ be a generic linear subspace of dimension $d$ in $\C^D,$ let $f_1,\dots, f_m\in \Id(V)$ be generic with degree one. Then
$\VS (f_1,\dots, f_m)$ is a generic linear subspace of dimension $\max (D-m, d)$ containing $V.$ In particular, if $m\ge D-d,$ then $\VS (f_1,\dots, f_m)=V.$ In this example, $U= \VS(f_1,\dots, f_m)$, if $m < \codim V,$ with codimension $m$, and $U=\varnothing$, if $m\ge \codim V,$ with codimension $D+1.$

Similarly, one may construct generic examples with arbitrary behavior for $\codim U$ when $m\ge \codim V,$ by choosing $V$ and the degrees of $f_i$ appropriately.
\end{Ex}

As in the geometric version for the height theorem, we may derive the following geometric interpretation of this result:
\begin{Cor}
Let $V\subseteq Z_1$ be fixed algebraic sets in $\C^D$. Let $Z_2$ be a generic algebraic set in $\C^D$ containing $V.$ Then
\begin{align*}
\codim (&Z_1 \cap Z_2 \setminus V)\ge\\
&\min (\codim (Z_1 \setminus V) + \codim (Z_2 \setminus V),\; D+1).
\end{align*}
\end{Cor}
Informally, we have derived a height theorem type result for algebraic sets under the constraint that they contain another prescribed algebraic set $V$. \\

We also want to give a homogenous version of Proposition~\ref{Prop:KrullHt-algset}, since the ideals in the paper are generated by homogenous forms:
\begin{Cor}\label{Cor:KrullHt-Hom}
Let $V$ be a fixed homogenous algebraic set in $\C^D$.
Let $f_1,\dots, f_m$ be generic homogenous forms in $\Id (V),$ satisfying the degree condition as in Proposition \ref{Prop:KrullHt-algset}. Then $\VS (f_1,\dots, f_m)=V+ U$ with $U$ an algebraic set fulfilling
$$\codim U\ge \min (m,\;D).$$
In particular, if $m> D,$ then $\VS (f_1,\dots, f_m)=V.$
Also, the maximal dimensional part of $\VS (f_1,\dots, f_m)$ equals $V$ if and only if $m > D- \dim  V.$
\end{Cor}
\begin{proof}\em
This follows immediately by dehomogenizing, applying Proposition~\ref{Prop:KrullHt-algset}, and homogenizing again.
\end{proof}

\subsection{Hilbert series of generic ideals}
In this section we will study the dimension of the vector spaces of homogenous polynomials of fixed degrees. A classical tool to do this in commutative algebra are Hilbert series; let us introduce some notations first.
\begin{Not}
We will write $R=\C[X_1,\dots, X_D].$
Let $\calI$ be some ideal of $R,$ or $R$ itself. We will denote the $\C$-vector space of homogenous polynomials of degree $k$ in $\calI$ by $\calI_k$.
\end{Not}
The Hilbert series links the dimensions of those vector spaces to the graded structure of the whole ideal:
\begin{Def}
Let $\calI$ be some ideal of $R.$ Then the {\it Hilbert series} of $\calI$ is the power series
$$H(\calI)(t)=\sum_{k=0}^\infty t^k \left(\dim(R_k) - \dim (\calI_k)\right).$$
\end{Def}
It is classically known that the $a_k$ satisfy a polynomial relation for $k\ge M$ with a big enough $M$. However, we will be mainly interested in the exact coefficients $a_k$ below $k\le M$ when the ideal $\calI$ is conditioned generic. I.e.~we are interested in the situation where we have some ideal $\fraks$, and an ideal $\calI$ generated by generic homogenous polynomials $f_1,\dots, f_m$ in $\fraks.$ Since in this situation, we have $\calI\subseteq \fraks$, we will consider the Hilbert series of the difference
\begin{align*}
H(\calI/\fraks)(t)&=H(\calI)(t)-H(\fraks)(t)\\
&=\sum_{k=0}^\infty t^k \left(\dim(\fraks_k) - \dim (\calI_k)\right).
\end{align*}
The following homogenous version of Proposition~\ref{Prop:dehom-rad-generic} will allow us to study this further:
\begin{Prop}\label{Prop:radquot}
Let $\fraks\subseteq R$ be a homogenously saturated ideal generated by $n$ homogenous elements of degree at most $\delta.$ Let
$$\calI=\langle f_1,\dots, f_m\rangle,\quad m\ge \max(D+1, n)$$
with generic $f_i\in \fraks$ such that
$$\deg f_i\ge \delta \quad \mbox{for all}\; 1\le i\le m.$$
For any $1\le j\le D$, we then have
\begin{align*}
\fraks&=(\calI : X_j)\\
&=\{g\in R\;:\; gX_D^n\in \calI\;\mbox{for some}\;n\in\N\}.
\end{align*}
\end{Prop}
\begin{proof}\em
Since the $f_i$ are generic, we may make a permutation of variables without altering the statement; i.e.~we may assume that $j=D.$
The proof strategy will be to derive a homogenous version of Proposition \ref{Prop:dehom-rad-generic}. In order to do this, we first dehomogenize every object with respect to $X_D$, i.e.~we substitute $1$ for $X_D$. Then, we will be in the situation of Proposition \ref{Prop:dehom-rad-generic} for the dehomogenized objects, and from that, we can conclude the statement for our homogenous version.\\

Let us first fix some notation: Let $g_1,\dots, g_{n}$ be some generators for $\fraks.$ Let $\fraks'$ be the dehomogenization of $\fraks$. The ideal $\fraks'$ is generated by $g_1',\dots, g_{n}'$ in the $D-1$ variables $X_1,\dots, X_{D-1}$, where $g_i'$ is the dehomogenization of $g_i$. The dehomogenization of $\calI$ is also an ideal in $D-1$ variables, generated by the dehomogenizations $f_i'$ of $f_i$. By Lemma \ref{Lem:dehom}, the $f_i'$ are generic polynomials in $\fraks'$, of same degrees as the $f_i$. \\

Now we are in the situation of Proposition \ref{Prop:dehom-rad-generic}: $\calI'$ is an ideal generated by the generic polynomials $f_i'$ in $\fraks'.$ We also have $\deg f_i'\ge \max_i\deg g_i'.$ Thus we may conclude that $\calI'=\fraks'.$\\

To prove the main statement from this, it now suffices to prove that $g_1\in (\calI : X_D),$ since the numbering of the $g_i$ is arbitrary, and thus it will then follow that $g_i\in (\calI: X_D)$ for any $i$, which implies $(\calI: X_D)\supseteq \fraks.$ On the other hand, as $\fraks$ is saturated, we have that $(\calI: X_D)\subseteq (\fraks: X_D)=\fraks,$ and thus have proved both inclusions, when seeing that $g_1\in (\calI : X_D).$

By our above reasoning, we have $\calI'=\fraks'$, so there exist polynomials $P_i\in\C[X_1,\dots, X_{D-1}]$ such that
$$g'_1=f'_1P_1+\dots + f'_m P_m.$$
Let $a=\deg g'_1$, let $d_i=\deg (f'_1 P_1)$, and let $d'=\max_i d_i.$ By polynomial arithmetic, we have $a\le d'.$ Let $Q_i$ be the homogenization of the $P_i.$ We then have
$$g_1 X_D^{d'-a}=f_1Q_1 X_D^{d'-d_1}+\dots + f_m Q_m X_D^{d'-d_m}.$$
The right hand side is an element of the ideal $\calI,$ thus the left hand side must be also in $\calI.$ In particular, this implies that $g_1\in (\calI : X_D)$, what had to be proven.

(Note that we have implicitly re-proved that the homogenization of the dehomogenization of an ideal is its homogenous saturation).
\end{proof}
Readers familiar with algebra may note that Proposition \ref{Prop:radquot} is only a description of the homogenization of the ideal $\fraks',$ respectively the homogenous saturation of the ideal $\fraks$. This is no surprise, since it is merely the homogenous reformulation of Proposition \ref{Prop:dehom-rad-generic}.\\

This Proposition directly implies that the coefficients of $H(\calI/\fraks)(t)$ stabilize to zero if $\calI$ has enough generators:

\begin{Prop}\label{Prop:multterm}
Let $f_i,\calI, \fraks$ be as in Proposition~\ref{Prop:radquot}.
Then there exists an $N\in \N$ such that
$$\calI_N=\fraks_N.$$
\end{Prop}
\begin{proof}\em
Let us fix a homogenous generating set $g_1,\dots, g_n$ for $\fraks$, let $\delta =\max_j\deg g_j.$ The set consisting of all elements $g_i M$ where $1\le i\le n$ and $\deg g_i\le k$ and $M$ a monomial in $X_1,\dots, X_D$ of degree $k-\deg g_i$ is a generating set for $\fraks_k$. By Corollary~\ref{Prop:radquot} we know that for each $i$ and each $j$, there exists a number $q_{ij}$ such that $g_iX_j^{q_{ij}}\in \calI.$ Let $q$ be the maximum of the $q_{ij}, 1\le i\le n, 1\le j\le D.$ Then $g_iX_j^q\in \calI$ for every $i,j$. Now by the pigeonhole principle, every monomial $M$ in $X_1,\dots, X_D$ of degree $D(q-1)+1$ will be divisible by $X_j^q$ for some $j.$ In particular, $g_iM\in \calI$ for every $i$ and every monomial $M$ of degree $D(q-1)+1$. In particular,
$$\fraks_N\subseteq \calI_N$$
for $N=\delta+D(q-1)+1,$ which proves the claim.
\end{proof}

For the case where instead of $\fraks$ we take the whole ring $\C[X_1,\dots, X_D],$ Fröberg's famous conjecture \cite{Fro94} states what the Hilbert function would be expected to be:

\begin{Conj}\label{Conj:Fro}
Let $f_1,\dots, f_m$ be generic homogenous polynomials in $R$ of fixed degrees $d_1,\dots, d_m,$ let
$\calI=\langle f_1,\dots, f_m\rangle.$ Then
$$H(\calI)(t)=\left| \frac{\prod_{i=1}^m (1-t^{d_i})}{(1-t)^D}\right|,$$
where for a power series, $\left|\sum_{k=0}^\infty a_kt^k\right|$ denotes setting all coefficients $a_\ell$ to zero
for which there exists $k$ such that $k<\ell$ and $a_k< 0.$
\end{Conj}

The Conjecture is known to be true for several cases, Fröberg has proved the following \cite{Fro85, Fro94}:
\begin{Thm}\label{Thm:Fro}
Let $f_1,\dots, f_m$ be any homogenous polynomials in $R$ of fixed degrees $d_1,\dots, d_m,$
let $\calI=\langle f_1,\dots, f_m\rangle.$ Let
$$H(\calI)(t)=\sum_{k=0}^\infty b_k t^k$$
be the true Hilbert series of $\calI$, and
$$\sum_{k=0}^\infty a_kt^k=\left|\frac{\prod_{i=1}^m (1-t^{d_i})}{(1-t)^D}\right|$$
the Hilbert series from Conjecture~\ref{Conj:Fro}. Then one has
$$b_k\ge a_k.$$
Equality holds if the $f_i$ are generic and $m\le D.$
\end{Thm}

In view of the evidence we have gathered in numerical computer experiments, we formulate the following generalization of Fröberg's conjecture~\ref{Conj:Fro}
for the conditioned case:
\begin{Conj}\label{Conj:Frogen}
Let $\fraks\subseteq R$ be a homogenous ideal, having a generating set in degree $\le \delta.$
Let $f_1,\dots, f_m$ be generic homogenous polynomials in $\fraks$ of fixed degrees $d_1,\dots, d_m\ge \delta.$
Let $\calI=\langle f_1,\dots, f_m\rangle.$ Then
$$H(\calI/\fraks)(t)=\left| \frac{\prod_{i=1}^m (1-t^{d_i})}{(1-t)^D}-H(\fraks)(t)\right|.$$
\end{Conj}

One can generalize Fröberg's theorem~\ref{Thm:Fro} to the conditioned case:
\begin{Thm}\label{Thm:Frogen}
Let $\fraks\subseteq R$ be a homogenous ideal, having a generating set in degree $\le \delta.$
Let $f_1,\dots, f_m$ be any homogenous polynomials in $R$ of fixed degrees $d_1,\dots, d_m\ge \delta,$
let $\calI=\langle f_1,\dots, f_m\rangle.$ Let
$$H(\calI/\fraks)(t)=\sum_{k=0}^\infty b_k t^k$$
be the true Hilbert series of $\calI/\fraks$, and
$$\sum_{k=0}^\infty a_kt^k=\left|\frac{\prod_{i=1}^m (1-t^{d_i})}{(1-t)^D}-H(\fraks)(t)\right|$$
the Hilbert series from Conjecture~\ref{Conj:Frogen}. Then one has
$$b_k\ge a_k.$$
Equality holds for $m\le d,$ where $d$
is the Krull dimension of $R/\fraks.$
\end{Thm}
\begin{proof}\em
For the first part, we use Fröberg's original theorem~\ref{Thm:Fro}. Let us denote
$$\sum_{k=0}^\infty c_kt^k=\left|\frac{\prod_{i=1}^m (1-t^{d_i})}{(1-t)^D}\right|.$$
The theorem then implies that
$$\dim \calI_k\le \dim R_k - c_k.$$
Also, since $\calI\subseteq \fraks,$ we have
$$\dim \calI_k\le \dim\fraks_k.$$
Translating this into differences to $\fraks$, we obtain
$$b_k=\dim \fraks_k-\dim \calI_k\ge \dim \fraks_k-\dim R_k + c_k$$
and
$$b_k=\dim \fraks_k-\dim \calI_k\ge 0.$$
On the other hand,
$$a_k=\dim \fraks_k-\dim R_k + c_k$$
for all $k$ until the right hand side would become negative, from where it is zero.
Together with the above, this implies $b_k\ge a_k.$

Now we will prove Conjecture~\ref{Conj:Frogen} for $m\le d$. We can assume that $\calI$ and $\fraks$ are in Noether position, i.e.~the chosen coordinate system $X_1,\dots, X_D$ is generic (with respect to unitary linear transformations). Since $d$ is the Krull dimension of $\fraks$, we may assume that $X_1,\dots, X_d$ are transcendental variables in $R/\fraks.$ Let
$\tilde{f}_1,\dots, \tilde{f}_m$ be the polynomials in $\C[X_1,\dots, X_D]$, obtained from setting all terms in $f_1,\dots, f_m$ to zero which are not divisible by one of the variables $X_1,\dots, X_d$. These generate an ideal $\tilde{\calI}\subseteq \C[X_1,\dots, X_D].$ Since the $X_1,\dots, X_d$ are transcendental in $R/\fraks,$ and $\calI$ and $\fraks$ are in Noether position, the remaining monomials are linearly independent. Thus we have that $\dim \tilde{\calI}_k\le \dim\calI_k = \dim \fraks_k - b_k.$
But the $\tilde{f}_1,\dots, \tilde{f}_m$ form a regular sequence, since the coefficients are generic, so we may use Fröberg's original theorem~\ref{Thm:Fro}, obtaining the correct dimension.
\end{proof}

These considerations give us important bounds on the $N$ from Proposition~\ref{Prop:multterm}:
\begin{Cor}\label{Cor:Nbound}
Let $\fraks\subseteq R$ be an ideal generated by $n$ homogenous elements of degree at most $\delta.$ Let
$f_1,\dots, f_m\in \fraks$ be generic with degrees $d_1,\dots, d_m\ge \delta$
and $m\ge \max(D+1, n).$ Let $\calI=\langle f_1,\dots, f_m\rangle.$
Let $N'$ be the smallest number such that the coefficient $a_{N'}$ in the power series
$$\sum_{k=0}^\infty a_kt^k=\frac{\prod_{i=1}^m (1-t^{d_i})}{(1-t)^D}-H(\fraks)(t)$$
is non-positive. Then
$$\calI_N=\fraks_N$$
only if $N\ge N'.$ If Conjecture~\ref{Conj:Frogen} holds, the converse is also true.
\end{Cor}

\subsection{An Algorithm to prove the generalized Fröberg conjecture}
In this subsection, we will present an algorithm with which one can use in a computer assisted proof of Conjecture~\ref{Conj:Frogen} for fixed $d_i, D$ and $\fraks.$

The basic observation is that given polynomials $f_1,\dots, f_m\in \fraks$ of fixed degrees $d_1,\dots, d_m$ and the ideal $\calI=\langle f_1,\dots, f_m\rangle,$ the assertion $A(c)=\left[\dim \calI_k \le c\right]$ is an algebraic property for every $c$, since it corresponds to the vanishing of (sub-)minors of a matrix whose coefficients can be expressed in those of $f_i$. Note that $A(c)$ depends on the $f_i$ resp.~$d_i,$ but for reading convenience we do not write that explicitly out. By Theorem~\ref{Thm:Frogen}, $A(\dim\fraks_k-a_k)$ (with $a_k$ as in the theorem) is the sure event resp.~the true property, which is also irreducible.

If we can now find a single set of polynomials $\tilde{f}_1,\dots, \tilde{f}_m\in\fraks$ of degrees $d_1,\dots, d_m$ for which $A(\dim\fraks_k-a_k)$ holds but not $A(\dim\fraks_k-a_k-1),$ this implies that $A(\dim\fraks_k-a_k)$ does not imply $A(\dim\fraks_k-a_k-1)$ or any of its irreducible sub-properties. Thus, by the definition of genericity, we would have proved that generic polynomials $f_1,\dots, f_m$ fulfill $A(\dim\fraks_k-a_k),$ but not $A(\dim\fraks_k-a_k-1).$ Checking this for all $k$ up to the $N$ for which $\calI_N=\fraks_N$ then proves Conjecture~\ref{Conj:Frogen} for the fixed set of $d_i, D$ and $\fraks.$

These considerations give rise to Algorithm~\ref{Alg:Frogen}, which can be used to prove Conjecture~\ref{Conj:Frogen} for specific sub-cases.

\begin{algorithm}[h]
\caption{\label{Alg:Frogen} Checking Conjecture~\ref{Conj:Frogen}.
\textit{Input:} Degrees $d_1,\dots, d_m$, number of variables $D$; the ideal $\fraks$.
\textit{Output:} Terminates if $b_k=a_k$ in Theorem~\ref{Thm:Frogen}.}
\begin{algorithmic}[1]
    \State \label{alg:Frostart} Randomly sample polynomials $f_1,\dots, f_m$ of degrees $d_1,\dots, d_m$ from $\fraks$

    \State Initialize $Q \gets [\,]$ with the empty matrix.

    \For{$i=1\dots n$}
       \For{all monomials $M$ of degree $k-\deg f_i$}
            \State

            	Add a row vector of coefficients, $Q \gets \begin{bmatrix}  Q \\ f_i M \end{bmatrix}$

       \EndFor
    \EndFor

	\State Calculate $r=\rk Q.$

    \If {$r = \dim\fraks_k-a_k$}
    \State Terminate
    \Else
    \State Goto step~\ref{alg:Frostart}
    \EndIf
\end{algorithmic}
\end{algorithm}

In order to check Conjecture~\ref{Conj:Frogen} for fixed $d_1,\dots,d_m$ and $\fraks$, one executes Algorithm~\ref{Alg:Frogen} for $k=1,2,\dots$ and stops when $\calI_N=\fraks_N.$ This terminates if Conjecture~\ref{Conj:Frogen} is true. It is important to note that the computations in Algorithm~\ref{Alg:Frogen} only constitute a proof if they are carried out exactly, or with floating point arithmetic where one additionally has to ensure that the initial numerical error cannot increase the rank $r$. This can be for example ensured by computing $r$ as the approximate rank of $Q$ with respect to a high enough threshold, depending on the machine precision and the propagation of initial errors.

We have checked Conjecture~\ref{Conj:Frogen} in the case where $\fraks$ is an ideal of dimension $d$ generated by $D-d$ linear forms, and the $f_i$ are quadrics - the simplest case relevant for the statistical marginalization problem. As the coordinate system can be regarded as generic, it suffices to check Conjecture~\ref{Conj:Frogen} for a specific choice of $\fraks$ where $d$ is fixed, as the genericity phenomena stay invariant under linear transformations.

\begin{Thm}\label{Thm:Frocheck}
Conjecture~\ref{Conj:Frogen} is true for linear $\fraks$, $d_1,\dots, d_m\le 2$ and $D\le 11.$
\end{Thm}
\begin{proof}\em
This follows from the above considerations and executing Algorithm~\ref{Alg:Frogen} for any of the finitely many possible cases. As the algorithm is correct and we have found that it terminates, Conjecture~\ref{Conj:Frogen} is true.
\end{proof}

Of course, Algorithm~\ref{Alg:Frogen} as it is cannot be used to prove Conjecture~\ref{Conj:Frogen} in total, as this would require to check a countably infinite number of cases for every $\fraks$. On the other hand, even if Conjecture~\ref{Conj:Frogen} does not hold, it can be slightly modified to Algorithm~\ref{Alg:FroN} which may be used for computing the $N$ in Proposition~\ref{Prop:multterm}.

\begin{algorithm}[h]
\caption{\label{Alg:FroN} Compute $N$ in Proposition~\ref{Prop:multterm}.
\textit{Input:} Degrees $d_1,\dots, d_m$, number of variables $D$; the ideal $\fraks$.
\textit{Output:} An $N$ such that $\fraks_N=\calI_N$.}
\begin{algorithmic}[1]
    \State Calculate $N'$ as in Corollary~\ref{Cor:Nbound}, $k\gets N'.$

    \State  \label{alg:FroNstart} Randomly sample polynomials $f_1,\dots, f_m$ of degrees $d_1,\dots, d_m$ from $\fraks$

    \State Initialize $Q \gets [\,]$ with the empty matrix.

    \For{$i=1\dots n$}
       \For{all monomials $M$ of degree $k-\deg f_i$}
            \State

            	Add a row vector of coefficients, $Q \gets \begin{bmatrix}  Q \\ f_i M \end{bmatrix}$

       \EndFor
    \EndFor

	\State Calculate $r=\rk Q.$

    \If {$r = \dim\fraks_k$}
    \State Terminate
    \Else
    \State $k\gets k+1,$ goto step~\ref{alg:FroNstart}
    \EndIf
\end{algorithmic}
\end{algorithm}

If the calculations are performed exactly, then the algorithm yields the smallest $N$ from Proposition~\ref{Prop:multterm}. If the calculations are performed in floating point arithmetic, it is not guaranteed that it finds the smallest $N$, but it terminates with probability one.

\section{Applications to ideal regression}
\label{sec:appidr}
In this section, we present some fundamental properties for the ideal regression problem which can be derived from the results on genericity in section~\ref{app-generic}. Recall our formulation for ideal regression, which was derived e.g.~as Problem~\ref{Prob:eqsid}:
\begin{Prob}
Let $\mathcal{F}_\theta$ be a parametric family of radical ideals in $\mathbb{C}[X_1,\dots, X_D].$
Given input polynomials $q_1,\dots, q_m$, estimate a regression parameter $\theta$ such that $\mathcal{F}_\theta$ is ``close'' to the inputs $q_1,\dots, q_m.$
\end{Prob}

Before continuing, will give a generative version of the problem. In ordinary regression, it is a common assumption that there exists a true regressor hyperplane, and the data are generatively sampled from this regressor hyperplane, with some centered noise added. This naturally generalizes to the following assumption on the inputs $q_i$ in ideal regression:

\begin{Ass}
There is a true regressor parameter $\theta$ and a true regressor ideal $\mathcal{F}_\theta.$ Moreover, for every polynomial $q_i$, we have that
$$q_i=f_i+\varepsilon_i,$$
where $f_i$ is a generic polynomial of some fixed degree $d_i$ in $\mathcal{F}_\theta,$ and $\varepsilon_i$ is some generic polynomial.
\end{Ass}
This formulation models the classical splitting of sampling randomness and error randomness: the randomness in $f_i$ models the sampling process, while the randomness in $\varepsilon_i$ represents the noise.

Also note that while this assumption is natural for the common marginals problem, it may be too narrow for the general case; a broader assumption would be that $f_i$ is a generic polynomial from an ideal or a class of ideals (e.g.~ideals with fixed Hilbert function or Krull dimension) contained in $\mathcal{F}_\theta.$ However, due to brevity, we restrict to this class of ideal regression problems for the rest of the exposition.

In the following, we will restrict to the homogenous case, which is basically equivalent to the inhomogenous case. Finding a generating set for a homogenously generated $\mathcal{F}_\theta$ corresponds to finding a H-basis for an inhomogenous $\mathcal{F}_\theta.$ Thus, in all what follows, the parametric family $\mathcal{F}_\theta$ will be homogenously generated, and the $q_i,f_i,\varepsilon_i$ will homogenous polynomial-valued. Note that since the $f_i$ and $\varepsilon_i$ are generic, the $f_i, q_i$ and the $\varepsilon_i$ are in fact polynomial-valued random variables.

Under these assumptions, the ideal regression problem can be expressed as follows:
\begin{Prob}\label{Prob:idreg}
Let $\mathcal{F}_\theta$ be a parametric family of homogenously generated, radical ideals in $\mathbb{C}[X_1,\dots, X_D].$
For $1\le i\le m,$ let $f_i$ be generic polynomials in $\mathcal{F}_\theta$ for some fixed ground truth $\theta$, let $\varepsilon_i$ be generic polynomials. Let
$$q_i=f_i+\varepsilon_i$$
the noisy inputs. Given $q_1,\dots, q_m,$ estimate $\theta$.
\end{Prob}

Having well-definedness, the immediate question is about identifiability: is there a consistent estimator for $\theta$ in the $q_1,\dots, q_m$? That means, is there an estimator, which converges to the true value $\theta$, when the number of i.i.d.~samples for each $q_i$ (simultaneously) goes to infinity, or alternatively, the variance of the noise terms $\varepsilon_i$ (simultaneously) go to zero? In particular, considering their number $m$ and the degrees $d_1,\dots, d_m$. One necessary condition is that $\theta$ can be uniquely calculated from the noise-free sample $f_1,\dots, f_m$. In the following subsections, we will give a necessary condition for the latter and a general estimation algorithm for the noisy case.

\subsection{Identifiability}
In this section, we study the identifiability of the ideal regression problem, in the formulation of Problem~\ref{Prob:idreg}. By definition, identifiability is given if and only if there exists a consistent estimator for $\theta$. Equivalently, identifiability holds if and only if there is a consistent estimator for a system of generators of $\mathcal{F}_\theta$. As stated above, a necessary condition for identifiability is that $\mathcal{F}_\theta$ is uniquely identifiable from $f_1,\dots, f_m$. We conjecture that this condition is also sufficient. What we can say about identifiability is the following weaker, but provable sufficient condition:
\begin{Prop}
Let $\calI=\langle f_1,\dots, f_m\rangle.$ If $\mathcal{F}_\theta=(\calI:X_D),$ then the (noisy) ideal regression Problem~\ref{Prob:idreg} is identifiable.
\end{Prop}
\begin{proof}\rm
The assumption $\mathcal{F}_\theta=(\calI:X_D)$ gives a rule of calculation to obtain $\mathcal{F}_\theta\theta$ for the noise-free case, and thus $\theta$ due to unique parameterization. It remains to show that this rule can be adapted to deal with noise. But this can be algorithmically done, as we will show in the next chapter by stating the approximate saturation algorithm~\ref{Alg:sat-approx}.
\end{proof}

Thus, to get a sufficient condition on identifiability of ideal regression, we can now check when we can obtain $\mathcal{F}_\theta$ from saturating the ideal generated by the $f_i$. Since the $f_i$ are generic, we can apply Proposition~\ref{Prop:radquot} to directly obtain an identifiability criterion:
\begin{Thm}\label{Thm:ident}
Keep the notations for ideal regression as stated in Problem~\ref{Prob:idreg}. Then, the true parameter $\theta$ is identifiable if
\begin{align*}
m\ge \max(D+1, n)\;\mbox{and}\\
\deg f_i\ge \delta \quad \mbox{for all}\; 1\le i\le m,
\end{align*}
where $n$ is the cardinality of an arbitrary H-basis of $\mathcal{F}_\theta$, and $\delta=\max_i \deg f_i.$
\end{Thm}
For the common marginals problem, Theorem~\ref{Thm:ident} can be used to obtain a more sharp identifiability criterion, by noticing that any linearly generated homogenous ideal has an H-basis of at most $D$ elements in degree one:
\begin{Cor}
Keep the notations of Problem~\ref{Prob:idreg}. Consider the ideal regression problem, where $\calF_S=\Id(S)$, for a $d$-dimensional sub-vector space $S$ of $\C^D$. If $m\ge D+1,$ then $S$ is identifiable, independent of the degrees of the $f_i$.
\end{Cor}
In particular, this corollary also implies identifiability for the noise-free version stated in Problem~\ref{Prob:SSA-alg}. For sake of clarity, we state the particular situation for the noise-free case resp.~the $f_1,\dots, f_m$:

\begin{Cor}
\label{Cor:IdentS}
Let $\calI=\langle f_1,\dots, f_m\rangle $ be an ideal generated by $m \geq D+1$ generic homogenous polynomials vanishing on a
linear $d$-dimensional subspace $S\subseteq \C^D$, let $\ell$ be any linear homogenous polynomial. Then
\begin{align*}
	\sqrt{\calI}=\Id (S) = (\calI:\ell).
\end{align*}
\end{Cor}
\begin{proof}\em
The rightmost equality is a direct consequence of Proposition~\ref{Prop:radquot} and the fact that the coordinate system in Proposition~\ref{Prop:radquot} is arbitrary. The leftmost equality follows from Proposition~\ref{Prop:multterm} and the fact that $h^N\in\Id (S)_N$ for any linear homogenous element $h$ of $\Id (S)$.
\end{proof}

\subsection{Calculating approximate saturations}
\label{sec:appsat}
In this section, we will present an algorithm which is able to estimate the parametric ideal $\mathcal{F}_\theta$ consistently when the conditions of the identifiability Theorem~\ref{Thm:ident} are fulfilled, thus completing the proof of the theorem. If the conditions of the theorem are fulfilled, then from Proposition~\ref{Prop:radquot}, we know that $\mathcal{F}_\theta$ can be obtained
as the homogenous saturation of $\calI=\langle f_1,\dots, f_n\rangle$. While this is a classical task in Computational Algebraic Geometry, we do not know the $f_i$, but only the $q_i$ which are endowed with noise. Thus, we will have to calculate the saturation approximately.

For this, we the following Algorithm~\ref{Alg:sat-approx} to compute homogenous saturations approximately.
\begin{algorithm}[h]
\caption{\label{Alg:sat-approx} Approximate homogenous saturation.
\textit{Input:} A homogenous ideal $\calI=\langle f_1,\dots, f_n\rangle$.
\textit{Output:} Homogenous generator set $G$ for the approximate saturation $(\calI : X_D)$. }
\begin{algorithmic}[1]
    \State \label{alg:satdet} Determine $N$ such that $\calI_N = (\calI : X_D)_N.$
    \State Initialize $Q \gets [\,]$ with the empty matrix, $G \gets\{\}.$

    \For{$i=1\dots n$}
       \For{all monomials $M$ of degree $N-\deg f_i$}
            \State

            	Add a row vector of coefficients, $Q \gets \begin{bmatrix}  Q \\ f_i M \end{bmatrix}$

       \EndFor
    \EndFor
		
	\For{$k=N\dots 2$}
        \State Set $G \gets G \cup$ an approximate row basis for $Q$
		\State Set $Q \gets \mbox{ReduceDegreeHom(}Q\mbox{)}$
	\EndFor

    \State \label{alg:satret} Return $G$ (or reduce it first)
\end{algorithmic}
\end{algorithm}
In step \ref{alg:satdet}, Algorithm~\ref{Alg:sat-approx} first needs to find an $N$ where the saturation coincides with the ideal. Such an $N$ exists, however to find it is not a trivial task. Here, one needs either knowledge on $\calI$ or genericity assumptions as a simple criterion. Then, it builds with $Q$ an approximate representation of $(\calI : X_D)_N.$ From this, the method ReduceDegreeHom, which can be found as Algorithm~\ref{Alg:ReduceDegHom}, constructs an approximate representation of $(\calI : X_D)_{N-1}.$ This can be repeated until reaching a $k$ for which $(\calI : X_D)_k$ is empty. The calculations have to be performed approximately in the sense that the principal components of the row-spans have to be considered with a suitable singular value threshold.

Algorithm~\ref{Alg:ReduceDegHom} estimates $\left((\calI : X_D)\cap \langle X_D\rangle\right)_N$ approximately and then divides out $X_D$ from the approximate basis representation. Again, one has to consider principal components of a suitable approximation threshold. A more detailed description for the special case of the marginalization problem can be found in the main body along the algorithms presented there.

\begin{algorithm}[h]
\caption{\label{Alg:ReduceDegHom} ReduceDegreeHom ($Q$).
\textit{Input:} Approximate basis for $(\calI:X_D)_k,$
given as the rows of the matrix $Q.$
\textit{Output:} Approximate basis for $(\calI:X_D)_{k-1},$
given as the rows of the matrix $A$}
\begin{algorithmic}[1]		
		\State
		\begin{minipage}[t]{11cm}
			Let $Q'$ $\gets$ the submatrix of $Q$ obtained by\\ removing all
    		columns corresponding to\\ monomials divisible by $X_D$
		\end{minipage}
		
		\State
		\begin{minipage}[t]{11cm}
			Compute $L \gets$  an approximate\\ left null space matrix of $Q'$
		\end{minipage}
				
		\State
		\begin{minipage}[t]{11cm}
            Compute $L' \gets$ an approximate\\ row span matrix of $L_i Q$
        \end{minipage}
		
		\State
		\begin{minipage}[t]{11cm}
		Let $L'' \gets $ the matrix obtained from $L'$ by\\ removing all columns corresponding to\\
				monomials not divisible by $X_D$
		\end{minipage}

    \State \label{alg:red_line8}
    	\begin{minipage}[t]{11cm}
           Compute $A \gets $ an approximate\\ row span matrix
	 	   of $L''$
       \end{minipage}
  	  	
\end{algorithmic}
\end{algorithm}

A similar strategy can be applied when computing saturations $(\calJ: f)$ for arbitrary ideals.

However, we refrain from further explanations, as the given version is already sufficient to prove the identifiability Theorem~\ref{Thm:ident}.

Using Corollary~\ref{Cor:IdentS}, we can even obtain the saturation $\fraks=\sqrt{\calI}=(\calI : \ell)$ more efficiently and stably under genericity conditions which we have for example in the ideal regression problem~\ref{Prob:idreg}. Namely, Corollary~\ref{Cor:Nbound} allows us to obtain $\fraks_N$ for some $N$ from $\calI.$ Then it suffices to saturate the ideal $\langle \fraks_N\rangle$ by any linear polynomial $\ell$. As any polynomial $\ell$ will yield the same saturation, so one can additionally simultaneously saturate with respect to multiple linear polynomials and then compare or average.

Finally, if we know $\fraks$ to be linear, or if we know the Hilbert function of $\fraks$, we know the exact dimensions of the approximate spans and kernels (e.g.~from Theorem~\ref{Thm:Fro}). Algorithms~\ref{Alg:rad-ssa-approx} and \ref{Alg:ReduceDeg} (found in the paper) additionally use this specific knowledge in order to compute the saturation more accurately.

Moreover, Corollaries~\ref{Prop:radquot} and~\ref{Cor:Nbound} guarantee correctness and termination of the algorithm under genericity conditions if the inputs are exact; if they are subject to noise, the output of the algorithm approaches the true solution with decreasing noise, or, alternatively, increasing number of i.i.d.~samples, equal in distribution to the $f_i$.

\bibliographystyle{plain}
{\small
\bibliography{./stable,./segmentation,./ssa_pubs,./algebra}}

\end{document}